\documentclass{article}

\usepackage[final]{neurips_2023}
\usepackage{natbib}

\setcitestyle{numbers,square}

\usepackage{algorithm}
\usepackage{algorithmic}
\usepackage{multirow}
\usepackage{graphicx}
\usepackage[utf8]{inputenc} 
\usepackage[T1]{fontenc}    
\usepackage{hyperref}       
\hypersetup{
colorlinks=True
}
\usepackage{url}            
\usepackage{booktabs}       
\usepackage{amsfonts}       
\usepackage{nicefrac}       
\usepackage{microtype}      
\usepackage{xcolor}         
\usepackage{graphicx}
\usepackage{subcaption}
\usepackage{amsthm}
\usepackage{amsmath}

\title{Beyond Myopia: Learning from Positive and Unlabeled Data through Holistic Predictive Trends}

\author{
Xinrui Wang\textsuperscript{}\footnotemark[1]\quad Wenhai Wan\textsuperscript{}\footnotemark[1]\quad Chuanxin Geng\textsuperscript{}\quad Shaoyuan Li\footnotemark[2]\quad Songcan Chen\footnotemark[2]\\
\small College of Computer Science and Technology, Nanjing University of Aeronautics and Astronautics\\
\small MIIT Key Laboratory of Pattern Analysis and Machine Intelligence
}

\begin{document}
 \bibliographystyle{abbrv}

\maketitle
\renewcommand{\thefootnote}{\fnsymbol{footnote}}
\footnotetext[1]{Equal contribution: Xinrui Wang <wangxinrui@nuaa.edu.cn> and Wenhai Wan <wwh35@nuaa.edu.cn>}
\footnotetext[2]{Corresponding authors: Songcan Chen <s.chen@nuaa.edu.cn> and Shaoyuan Li <lisy@nuaa.edu.cn>.}
\renewcommand{\thefootnote}{\arabic{footnote}}

\begin{abstract}

Learning binary classifiers from positive and unlabeled data (PUL) is vital in many real-world applications, especially when verifying negative examples is difficult. Despite the impressive empirical performance of recent PUL methods, challenges like accumulated errors and increased estimation bias persist due to the absence of negative labels. In this paper, we unveil an intriguing yet long-overlooked observation in PUL: \textit{resampling the positive data in each training iteration to ensure a balanced distribution between positive and unlabeled examples results in strong early-stage performance. Furthermore, predictive trends for positive and negative classes display distinctly different patterns.} Specifically, the scores (output probability) of unlabeled negative examples consistently decrease, while those of unlabeled positive examples show largely chaotic trends. Instead of focusing on classification within individual time frames, we innovatively adopt a holistic approach, interpreting the scores of each example as a temporal point process (TPP). This reformulates the core problem of PUL as recognizing trends in these scores. We then propose a novel TPP-inspired measure for trend detection and prove its asymptotic unbiasedness in predicting changes. Notably, our method accomplishes PUL without requiring additional parameter tuning or prior assumptions, offering an alternative perspective for tackling this problem. Extensive experiments verify the superiority of our method, particularly in a highly imbalanced real-world setting, where it achieves improvements of up to $11.3\%$ in key metrics. The code is available at \href{https://github.com/wxr99/HolisticPU}{https://github.com/wxr99/HolisticPU}.
\end{abstract}

\section{Introduction}

Positive and Unlabeled Learning (PUL) is a binary classification task that involves limited positive labeled data and a large amount of unlabeled data \cite{liupu}.  This learning scenario naturally arises in many real-world applications like matrix completion\cite{hsieh}, deceptive reviews detection\cite{ren}, fraud detection\cite{outlierdetection} and medical diagnosis\cite{yang}.  It also serves as a key component of more complex machine learning problems, such as out-of-distribution detection\cite{zhou2021step} and adversarial training\cite{guo2020positive}. Two main categories of PUL methods are cost-sensitive methods and sample-selection methods. However, both approaches face their challenges. The cost-sensitive methods rely on the negativity assumption, which may introduce estimation bias due to the mislabeling of positive examples as negative\cite{noisypu}. This bias can be accumulated and even worsen during later training stages, making its elimination challenging. The sample-selection methods struggle with distinguishing reliable negative examples, particularly during the initial stage, which also results in error accumulation during the training process\cite{twosteprelabeling,twostepbase}.

As a basic component for various PUL methods, resampling the positive labeled data shows its potential in alleviating the bias brought by negative assumption \cite{noisypu,mixpul,kato2019learning,liyour,zhaodist}. For example, \cite{kato2019learning} resamples positive examples according to the given class prior and assumed label mechanism to achieve decent performance. In this paper, we dive deeper into this class of strategies. Instead of relying on one single-step prediction which is prone to model uncertainty, we take a holistic view and examine the predictive trend of unlabeled data during the training process. Specifically, we treat the unlabeled data as negative. In each training epoch, we resample over the labeled positive data to ensure a balanced class distribution. We evaluate the model's performance on CIFAR10 and FMNIST datasets\cite{cifar,xiaofashion} with 4 experimental settings. Our pilot experiments show that this resampling method achieves comparable or even state-of-the-art test performance at the outset, but underperforms soon after. Furthermore, the averaged predicting scores (output probability) of unlabeled negative examples exhibit a consistent decrease, whereas those of unlabeled positive examples display an initial increase before subsequent decreasing or oscillating. Conclusively, the averaged predictive trends for different classes exhibit significant differences, as depicted in Figure \ref{trend}. One possible explanation for these observations is the model's early focus on learning simpler patterns, which aligns with the early learning theory of noisy labels \cite{early}. Although the resampling strategy enjoys these advantages, selecting an appropriate model can be more challenging than the classification task itself due to the lack of a precise validation set. 

To break the above limitation, we propose a novel approach that treats the predicting scores of each unlabeled training example as a temporal point process (TPP). It takes a holistic view and surpasses existing methods that focus on examining loss values or tuning confidence thresholds based on a limited history of predictions. By centering on the difference in trends of predicting scores, our approach provides a more comprehensive understanding of deep neural network training in PUL. To further investigate whether this difference in trends is prevalent in individual unlabeled examples, we apply the Mann-Kendall Test, a non-parametric statistical test used to detect trends in the temporal point process \cite{hamed1998modified}, to the continuously predicting scores of each example. These scores are classified into three types: \textit{Decreasing}, \textit{Increasing}, and \textit{No Trend}. The statistical test reveals a clear distinction in the trends of predicted scores for each positive and negative example, supporting our observation. Our findings suggest that utilizing the model's classification ability in the early stages may be sufficient for successfully classifying unlabeled examples. This discovery offers us a new perspective on reformulating the problem of distinguishing positive and negative examples in the unlabeled set as identification of their corresponding predictive trends.

\begin{figure}[t]
  \centering
  \begin{subfigure}[b]{0.49\textwidth}
    \centering
    \includegraphics[width=\textwidth]{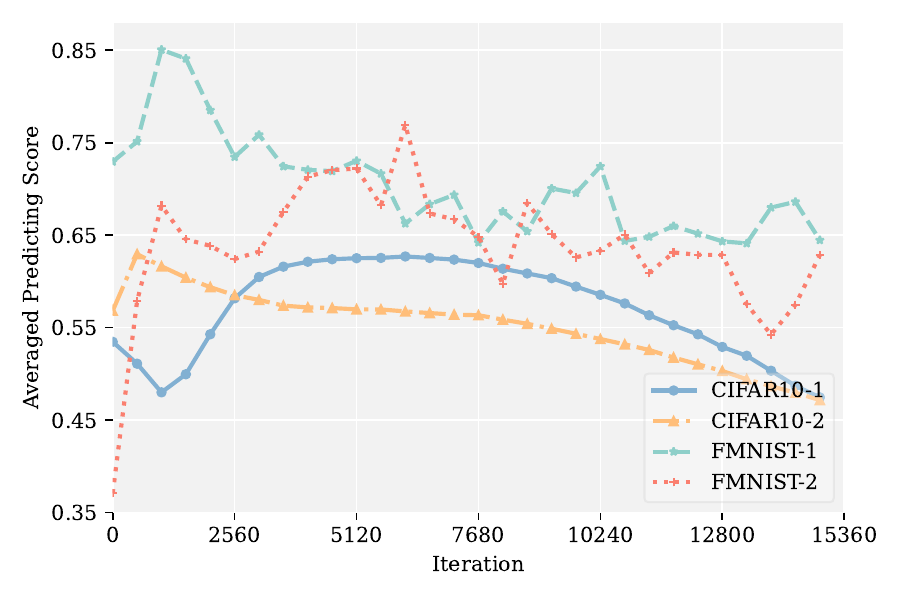}
    \label{positive_trend}
  \end{subfigure}
  \hfill
  \begin{subfigure}[b]{0.49\textwidth}
    \centering
    \includegraphics[width=\textwidth]{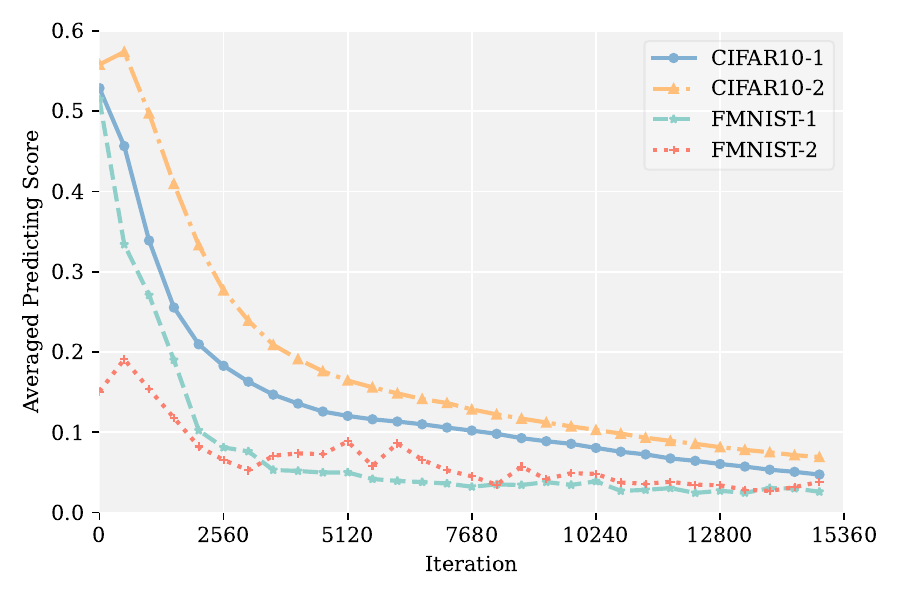}
    \label{negative_trend_statistic}
  \end{subfigure}
  \hfill
\caption{Averaged predicting scores (output probability) of positive (left) and negative (right) examples in an unlabeled dataset during the first 15,360 iterations of training (30 epochs).}
\label{trend}
\vspace{-18pt}
\end{figure}

We then propose a novel TPP-inspired measure, called \textbf{trend score} to quantify the distinctions in predictive trends. It is obtained by applying a robust mean estimator \cite{catoni2012challenging} to the expected value of the ordered difference in a TPP (sequence of predicting scores for each example)\cite{hamed2008trend}. Subsequently, we introduce a modified version of Fisher's Natural Break to distinguish these predictive trends, identifying a natural break point in the distribution of \textbf{trend score}. This approach divides examples into two groups: the group with \textbf{high trend score} represents positive examples, while the group with \textbf{low trend score} corresponds to negative examples. Our approach simplifies the training process by circumventing threshold selection when assigning pseudo-labels. Once the unlabeled data is classified, the remaining problem becomes a binary supervised learning task, and issues such as estimating class priors can be easily addressed. In summary, our main contributions are:

\begin{itemize}
\item We demonstrate the effectiveness of the proposed resampling strategy. It is also observed that predictive trends for each example can serve as an important metric for discriminating the categories of unlabeled data, providing a novel perspective for PUL.

\item We propose a new measure, \textbf{trend score}, which is proved to be asymptotically unbiased in the change of predicting scores. We then introduce a modified version of Fisher's Natural Break with lower time complexity to identify statistically significant partitions. This process does not require additional tuning efforts and prior assumptions.

\item We evaluate our proposed method with various state-of-the-art approaches to confirm its superiority. Our method also achieves a significant performance improvement in a highly imbalanced real-world setting.
\end{itemize}
\vspace*{-6pt}
\section{Our Intuition and Method}
\vspace*{-6pt}
\subsection{Preliminary}
We first revisit some important notations in PUL. Formally, let $x\in \mathbb{R}^d$ be the input data with $d$ dimensions and $y\in \{0,1\}$ be the corresponding label. Different from the traditional binary classification, PUL dataset is composed of a positive set $\mathcal{P}=\{x_i,y_i=0\}^{n_p}_{i=1}$ and an unlabeled set $\mathcal{U}=\{x_i\}^{n_u}_{i=1}$, where the unlabeled set $\mathcal{U}$ contains both positive and negative data. Throughout the paper, we denote the positive class prior as $\pi = \mathbb{P}(y = 0)$.
\vspace{-5pt}
\subsection{Resampling Strategies for Positive and Unlabeled Learning}
\begin{figure}[ht]
\vspace{-4pt}
    \centering
    \includegraphics[scale=0.303]{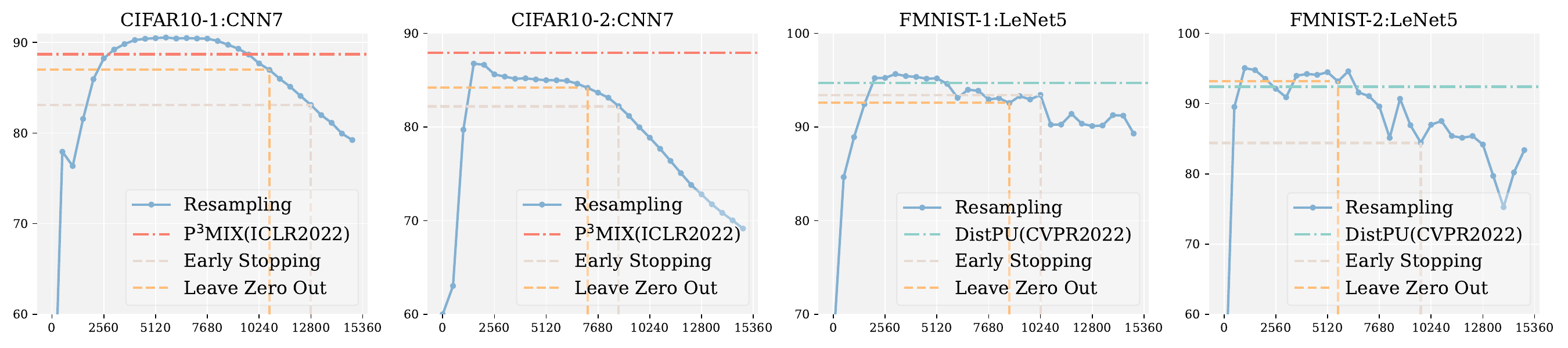} 
    \caption{The accuracy of our resampling method (first 30 epochs). The horizontal line represents the accuracy of the state-of-the-art methods. Early stopping and Leave Zero Out represent different model selection strategies.}
    \label{performance}
\vspace{-4pt}
\end{figure}
Resampling strategies have long been a baseline for dealing with imbalanced data or limited labels, which naturally fits PUL since its key challenge lies in limited labels and potentially imbalanced data distribution\cite{chawla2002smote}. Different from popular resampling strategies applied in PUL\cite{kato2019learning}, we follow the training scheme as \cite{fixmatch,flexmatch} to independently sample positive and unlabeled data as different data batches and the loss function is defined accordingly.
\begin{equation}
\mathcal{L}=\frac{1}{|\mathcal{B}p|}\sum_{(x_i,y_i)\in \mathcal{B}_p}\ell(\hat{y_i},y_i) + \frac{1}{|\mathcal{B}u|}\sum_{x_i\in \mathcal{B}_u}\ell(\hat{y_i},1), \hspace{0.4cm}
\hat{y_i} = f(x_i).
\label{loss}
\end{equation}
Here, we denote $f\in\mathcal{F}$ as a binary classifier, $\ell(\cdot,\cdot)$ as the loss function, $\mathcal{B}_p$ and $\mathcal{B}_u$ as the positive and unlabeled training batches respectively. We ensure that $|\mathcal{B}_p|=|\mathcal{B}_u|$ to achieve a balanced class prior during the training process. This approach emphasizes the labeled data and mitigates the imbalance of positive and pseudo-negative labels, which also provides a good theoretical explanation when dealing with high-dimensional data conforming to different Gaussian distributions. As shown in Appendix\ref{decision hyperplane}, an optimal decision hyperplane can be attained when $|\mathcal{P}|/|\mathcal{U}|$ equals 1. Figure\ref{performance} details the performance of our resampling baseline on two datasets under four different settings. It can be observed that the proposed method performs comparably or even better than state-of-the-art methods (P$^3$MIX\cite{liyour} and DistPU\cite{zhaodist}) in the early stages of training, as demonstrated by its test performance at certain epochs. However, the method's performance quickly degrades in all 4 settings as the estimation bias worsens during training due to the false negatives introduced by the negativity assumption. We also explore alternative model selection strategies, such as holding out a validation set from given labeled examples or using different versions of augmented data for model selection, as inspired by prior studies \cite{leave, mahsereci2017early}. In addition to the common practice of selecting the model from an additional positive validation set, we also implement LZO\cite{leave}, which selects the model based on the mixup-induced validation set. As shown in Table\ref{tab_resample}, the performance gap persists, especially when most of the unlabeled data belongs to the positive class.

\begin{table*}[htbp]
\caption{Classification accuracy (Recall rate is reported on Credit Card) on unlabeled training data. Resampling-P represents the model selected on an extra positive validation set. Resampling-LZO represents the model selected through LZO. Resampling* represents the best model selected on the test set which is an ideal case.} 
\small
\centering
\setlength\tabcolsep{5pt}
\resizebox{1\textwidth}{!}
{
    \begin{tabular}{ccccccccc}
    \toprule[1.5pt]
    Dataset & F-MNIST-1 & F-MNIST-2 & CIFAR10-1 & CIFAR10-2 & STL10-1 & STL10-2 & Credit Card & Alzheimer\\ \midrule[1pt]
    Resampling-P & 89.93 &84.29 &81.06 &72.93 &- &- &60.75 &70.09\\
    Resampling-LZO & 93.37 &92.04 &84.87 &82.98 &- &- &67.24 &74.11\\
    Resampling*& 94.92 &94.57 &89.56 &85.46 &- &- &87.54 &76.30\\ 
     P$^3$MIX-C  & 91.59 &87.65 &86.05 &88.14 &- &- &76.21 &68.01\\ \bottomrule[1.5pt]
    \end{tabular}
}

\label{tab_resample}
\end{table*}

\label{early stopping} 
 

To tackle the above issues, some denoising-based semi-supervised PUL methods, such as \cite{selfpu, mixpul, noisypu}, have leveraged some threshold tuning or sample selection techniques to achieve acceptable empirical performance. These techniques have been criticized in \cite{xia2021sample} for relying solely on prediction scores or loss values, as they do not account for uncertainty in the selection process. This becomes even more problematic in PUL, where the noise ratio is typically higher when making a negativity assumption\cite{bekker}. 

To break the above limitations, we record the whole predicting process of each unlabeled training example to take a holistic view of the training. It is evident that averaged model-predicting scores for positive and negative data display two distinct trends when implementing the above resampling strategy in the early training stages. Meanwhile, the standard deviation of predictions for positive examples increases rapidly during training, making it increasingly difficult to select an appropriate threshold for distinguishing between positive and negative examples. The appropriate threshold interval for discriminating positive and negative examples quickly shrinks as training progresses, indicating that existing denoising techniques cannot fundamentally alleviate the issues of accumulated errors and increased estimation bias. Therefore, a more robust evaluation measure is necessary beyond relying on raw model-predicted scores or loss values. Implementation details in model selection and visualizations of threshold tuning are provided in Appendix\ref{resampling model}.

\subsection{Identifying Predictive Trends: A Key to Successful Classification}
\vspace{-5pt}
\label{sectrend}
\begin{figure}[ht]
    \centering
    \includegraphics[scale=0.303]{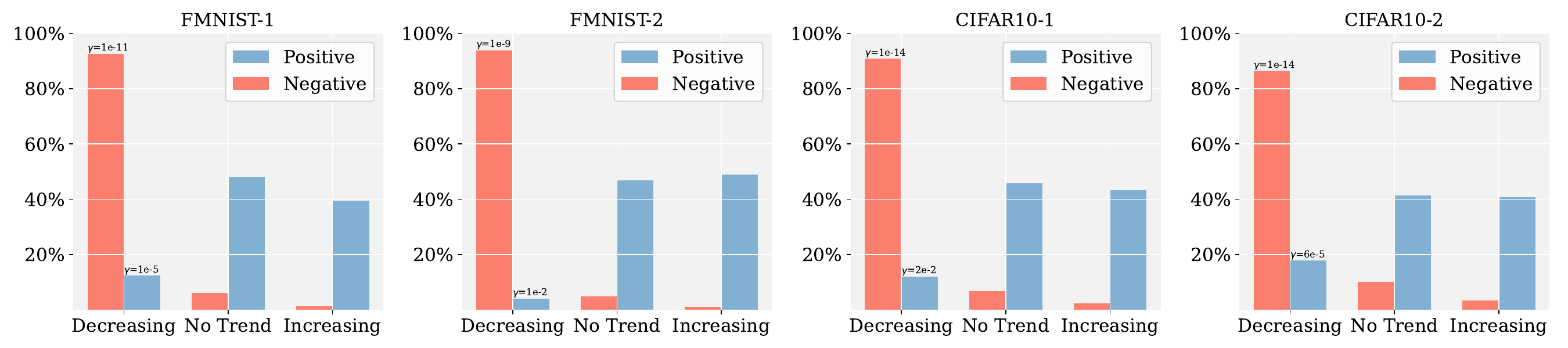} 
    \caption{The Mann-Kendall Test is performed on 4 settings of CIFAR10 and FashionMnist datasets. The figure reports the fractions of positive and negative examples in an unlabeled dataset exhibiting different predictive trends during the early training stage (first 30 epochs).}
    \label{trends_statistic}
\end{figure}

While deep neural networks have strong learning capabilities, they are at risk of overfitting all provided labels, regardless of their correctness. This can result in all unlabeled examples being predicted as negative \cite{arpit2017closer, zhang2021understanding}. We expect the predictive scores of negative examples in the unlabeled set to consistently decrease because all negative examples are given true negative labels by the negativity assumption. On the other hand, the predictive scores of positive examples in the unlabeled training set may not decrease initially because the resampled labeled examples are consistently emphasized from the start of training. To provide more evidence, we use the Mann-Kendall test to analyze the model-predicted scores of each example \cite{hamed1998modified}. This test categorizes the prediction sequence into three situations: \textit{Decreasing}, \textit{Increasing}, and \textit{No Trend}. The calculation process of the Mann-Kendall Test is detailed in Appendix\ref{MK test}. Figure \ref{trends_statistic} shows a contrast between the trends of predicted scores for positive and negative examples. Even when certain positive and negative examples exhibit a similar trend of decreasing prediction scores during training, we observed significant differences in the significance index $\gamma$ across different classes.

Our next objective is to measure the differences between positive and negative examples. To accomplish this, we require an evaluation measure that captures the significance of the observed trends in model-predicted scores. Before developing our own measure, an important notation in the TPP is first introduced, $\mathbb{E}[\Delta p]$, which represents the expected value of the ordered difference in a series of predicting scores. 
\begin{equation}
\vspace{-5pt}
\mathbb{E}[\Delta p]=\lim_{t\to\infty}\frac{2}{t(t-1)}\sum_{i<j}^{t}\Delta p_{ij}, \hspace{0.1cm}\Delta p_{ij}=p_j-p_i.
\end{equation} 
where $p_i$ is the predicting score (output probability) at $i$-th epoch, $t$ is the number of training epochs. 
\begin{equation}
\Tilde{S} = \frac{2}{t(t-1)}\sum_{i=1}^{t-1} \sum_{j=i+1}^{t} \Delta p_{ij}, \hspace{0.1cm}\Delta p_{ij}=p_j-p_i.
\end{equation} 
While $\Tilde{S}$ is the empirical mean and unbiased estimation of $\mathbb{E}[\Delta p]$, it can be unreliable for non-Gaussian examples and may not handle outliers or heavy-tailed data distributions well as illustrated in\cite{catoni2012challenging}. To address these issues, we propose a robust mean estimator inspired by\cite{xia2021sample, hamed1998modified}, called the \textbf{trend score} $S$, which measures the difference between each ordered pair of prediction scores:
\vspace{-0.1cm}
\begin{equation}
\hat{S} = \displaystyle\frac{2}{t(t-1)}\sum_{i=1}^{t-1} \sum_{j=i+1}^{t} \psi(\alpha\Delta p_{ij}),\hspace{0.1cm}\Delta p_{ij}=p_j-p_i.
\label{eq_trendscore}
\end{equation}
\begin{equation}
\psi(\Delta p_{ij}) =sign(\Delta p_{ij}) \cdot  log(1+|\Delta p_{ij}|+\Delta p_{ij}^2/2).
\end{equation}
\vspace{0.05cm}
in which $\alpha>0$ is a scaling parameter, and $sign()$ is the sign function that returns $-1$ if its argument is negative, $0$ if its argument is zero, and $1$ if its argument is positive. The function $\psi()$ can result in a more robust estimation by flattening the values of $\Delta p_{ij}$ and reducing the influence of minority outlier points on the overall estimation. Besides, we also provide a simplified version as:
\begin{equation}
    \dot{S} = \frac{1}{t-1}\sum_{i=1}^{t-1}  \psi(\alpha\Delta p_{ij}), \hspace{0.1cm}\Delta p_{ij}=p_j-p_i.
    \label{eq_simplified_trendscore}
\end{equation}
Notably, $\Tilde{S},\hat{S},\dot{S}$ are all calculated on each example. Experiments show that both $\hat{S},\dot{S}$ exhibit better empirical results than $\Tilde{S}$ in Section\ref{expriment}. For choosing the stopping epoch $t$, we implement the LZO\cite{leave} algorithm as described in Section\ref{early stopping}. We also derive a concentration inequality between our \textbf{trend score} $\hat{S}$ and the expected value of the ordered difference $\mathbb{E}[\Delta p]$.
\newtheorem{theorem}{Theorem}[section]
\begin{theorem}
Let $P=\{p_{ij}|1\leq i\leq t-1, 2\leq j \leq t,i<j\}$ be an observation set of changes in predictions in which $\mathbb{E}[\Delta p]$ is the expected values of the ordered difference in a temporal point process and $\sigma^2$ is the variance of $P$. By exploiting the non-decreasing influence function $\psi()$, for any $\epsilon>0$, we have the following bound with probability at least $1-2\epsilon$:
\begin{equation}
        |\hat{S}-\alpha\mathbb{E}[\Delta p]|<\frac{2\alpha\sigma\sqrt{\frac{2log(\epsilon^{-1})}{t(t-1)}}}{1-\sqrt{\frac{2log(\epsilon^{-1})}{t(t-1)\alpha^2\sigma^2}}}=\textit{O}\Big(\big({log(\epsilon^{-1}})\big)^\frac{1}{2}t^{-1}\Big).
    \end{equation}
\end{theorem}
It illustrates that the measure we propose is an asymptotically unbiased estimation with a linear weighting of $\mathbb{E}[\Delta p]$. The proof is provided in Appendix\ref{proof}. It is also proved in \cite{catoni2012challenging} that the deviations of this robust mean estimator can be of the same order as the deviations of the empirical mean computed from a Gaussian statistical sample, which further verifies the advantage of this estimator. 

\subsection{Clustering Unlabeled Data by the Fisher Criterion}
\vspace{-5pt}
\label{sec pl}
The topic of accurately labeling unlabeled data is widely discussed in various fields, including PUL. In the existing literature, threshold-based criteria and small loss criteria are the two primary approaches used for selecting reliable or clean examples, as seen in studies such as \cite{fixmatch, flexmatch, dividemix, noisypu}. However, previous works generally select examples based solely on current predictions, ignoring the inherent uncertainty in training examples, leading to longer training times and poor generalization ability\cite{xia2021sample, moore1990uncertainty}. Besides, they often require extensive hyperparameter tuning efforts to choose appropriate thresholds or ratios for data selection. In this section, we introduce a new labeling approach based on our proposed \textbf{trend score} tackling the above issues.

Our proposed \textbf{trend score} is the naturally comparable one-dimensional data and allows the Fisher Criterion to be a viable choice. It identifies a natural break point in the trend score distribution, which could be used to divide the data into two groups: one with high trend scores and one with low trend scores representing positive and negative examples respectively. Specifically, the objective function of finding this Fisher's natural break point can be formed as follows:
\begin{equation}
\label{eq_cluster}
    \begin{aligned}
    &\min_{C_1, C_2}  \displaystyle\frac{\sum_{x \in C_1} (\hat{S}_x - \mu_1)^2}{|C_1|} + \displaystyle\frac{\sum_{x \in C_2} (\hat{S}_x - \mu_2)^2}{|C_2|}\\
    &s.t.\hspace{0.1cm}C_1 \cap C_2 = \emptyset, \hspace{0.1cm} C_1 \cup C_2 = {x_1, x_2, \ldots, x_N}.
    \end{aligned}
\end{equation}
where $\hat{S}_x$ is our derived \textbf{trend score} for example $x$, $C_1$ and $C_2$ are the two clusters, $\mu_i$ is the mean of cluster $C_i$, and $N$ is the total number of data points. We utilize the Fisher natural break point method to automatically determine a threshold value that divided the trend score distribution into two distinct groups. Our implementation introduces an improved algorithm, which reduces the time complexity from $\textit{O}\big(N^2\big)$ to $\textit{O}\big(Nlog(N)\big)$, as explained in Appendix\ref{Natural Breaks Classification}. This method eliminates the need for manual threshold selection or hyperparameter tuning, both of which can be time-consuming and error-prone. Furthermore, the data-driven approach we used optimizes the threshold value for the specific dataset under analysis, rather than relying on arbitrary or pre-defined values. 

Once the unlabeled data is classified, the remaining task becomes a straightforward supervised learning problem. We directly train by a cross-entropy loss on the estimated labels given by Eq.\ref{eq_cluster} on the backbone network given in Table\ref{tab_dataset}. Besides, issues such as estimating class priors can be addressed easily when unlabeled data are classified.
\vspace*{-6pt}
\section{Experiments}
\vspace{-5pt}
\label{expriment}
\subsection{Classification on Unlabeled Training Set}
\vspace*{-5pt}
In this subsection, we first evaluate the performance of our method on the unlabeled training set compared with some state-of-the-art methods. As shown in Table\ref{tab_transductive}, our method demonstrates excellent classification performance on the unlabeled training data (the true labels of unlabeled data are not available in STL10). Moreover, a comparison with state-of-the-art prior estimation methods in PUL is conducted to further verify the effectiveness of our approach, and the results are presented in Table\ref{tab_prior estimate}.

\begin{table*}[htbp]
\caption{Classification accuracy (Recall rate is reported on Credit Card) on unlabeled training data.} 
\small
\centering
\setlength\tabcolsep{4pt}
\resizebox{1\textwidth}{!}
{
    \begin{tabular}{ccccccccc}
    \toprule[1.5pt]
    Dataset & F-MNIST-1 & F-MNIST-2 & CIFAR10-1 & CIFAR10-2 & STL10-1 & STL10-2 & Credit Card & Alzheimer\\ \midrule[1pt]
    nnPU        & 85.31 &82.46 &83.11 &83.23 &- &- &62.53 &64.01\\
    PGPU        & 92.02 &90.17 &85.67 &88.38 &- &- &42.12 &75.09\\
    Self-PU     & 94.04 &91.59 &84.06 &83.77 &- &- &71.00 &70.05\\
    P$^3$MIX-C  & 91.59 &87.65 &86.05 &88.14 &- &- &76.21 &68.01\\
    Ours        & \textbf{95.41} &\textbf{96.00} &\textbf{91.42} &\textbf{91.17} &- &- &\textbf{98.90} &\textbf{75.13}\\ \bottomrule[1.5pt]
    \end{tabular}
}
\vspace{-8pt}
\label{tab_transductive}
\end{table*}
\begin{table*}[htbp]
\caption{Absolute estimation error with the true positive prior in the first row. We implement an oracle early stopping for the extant methods as defined in \cite{mpe}. Our method significantly reduces estimation error when compared with existing methods.} 
\small
\centering
\setlength\tabcolsep{4pt}
\resizebox{1\textwidth}{!}
{
    \begin{tabular}{ccccccccc}
    \toprule[1.5pt]
    Algorithm & F-MNIST-1 & F-MNIST-2 & CIFAR10-1 & CIFAR10-2 & STL10-1 & STL10-2 & Credit Card & Alzheimer\\ \midrule[1pt]
    $\pi$  & 0.40 & 0.60 & 0.40 &0.60 &0.50 &0.50 &0.05 &0.50\\
    KM2         &0.146 &0.106 &0.115 &0.164 &0.096 &0.101 &0.236 &0.094\\
    BBE*        &0.082 &0.073 &0.034 &0.059 &0.046 &0.064 &0.112 &0.026\\
    (TED)$^n$   &0.026 &0.020 &0.042 &0.044 &0.024 &0.021 &0.018 &0.014\\
    Ours        &\textbf{0.014} &\textbf{0.021} &\textbf{0.016} &\textbf{0.031} &\textbf{0.018} &\textbf{0.009} &\textbf{0.004} &\textbf{0.011}\\ \bottomrule[1.5pt]
    \end{tabular}
}
\label{tab_prior estimate}
\vspace{-8pt}
\end{table*}

\subsection{Test Performance}
\begin{table*}[htb] 
\caption{Dataset description and corresponding backbones.}
\centering
\setlength\tabcolsep{24pt}
\small
{
    \resizebox{0.95\textwidth}{!}
    {
        \begin{tabular}{ccccc}
        \toprule[1.5pt]
        Dataset  & \#Trainset & \#Testset & Input size & Backbone    \\ 
        \midrule[1pt]
        F-MNIST  & 60,000     & 10,000    & 28$\times$28      & LeNet-5    \\
        CIFAR-10 & 50,000     & 10,000    & 3$\times$32$\times$32      & 7-Layer CNN \\
        STL-10   & 105,000    & 8,000     & 3$\times$96$\times$96      & 7-Layer CNN \\ 
        Alzheimer& 5,890      & 1,279     & 3$\times$224$\times$224  & ResNet-50 \\
        Credit Fraud &8,392    & 2098     & 30  & 6-Layer MLP \\
        \bottomrule[1.5pt]
        \end{tabular}
    }
}
\label{tab_dataset}
\vspace*{-10pt}
\end{table*}
We use three synthetic prevalent benchmark datasets including FashionMnist (F-MNIST) \cite{xiaofashion}, CIFAR10 \cite{cifar} and STL10 \cite{stl} and two real-world datasets on fraud detection\footnote{https://www.kaggle.com/datasets/mlg-ulb/creditcardfraud} and Alzheimer diagnosis\footnote{https://www.kaggle.com/datasets/tourist55/alzheimers-dataset-4-class-of-images} as our test set. We provide the dataset description and corresponding backbones in Table\ref{tab_dataset}, and the positive priors of each setting are given in Table\ref{tab_prior estimate}. More detailed description of benchmark datasets, dataset split and implementation details are given in Appendix\ref{implementaion detail}. For each dataset, we run our method for $5$ times with different random seeds and report the averaged classification accuracy. We follow the settings in \cite{mixpul,zhaodist}  when making the comparison: randomly select $769$ positive examples in Alzheimer dataset, $100$ positive examples in Credit Fraud dataset and $1000$ positive examples in others as the labeled set in training. Classification accuracy on test sets is reported as the main criterion. For highly imbalanced distributed (Credit Fraud) and biasedly selected (Alzheimer) datasets, we provide additional metrics such as Recall, F1 score and AUC on test sets for a more comprehensive comparison. 

\begin{table*}[htbp]
\caption{Results of classification accuracy ($\%$) on 3 generic datasets with 6 settings (mean±std).} 
\small
\centering
\setlength\tabcolsep{10pt}
\resizebox{0.95\textwidth}{!}
{
    \begin{tabular}{ccccccc}
    \toprule[1.5pt]
    Algorithm    & F-MNIST-1 & F-MNIST-2 & CIFAR10-1 & CIFAR10-2 & STL10-1 & STL10-2 \\ \midrule[1pt]
    uPU        & 81.6$\pm$1.2  & 85.7$\pm$2.6  & 76.5$\pm$2.5   & 71.6$\pm$1.4   & 76.7$\pm$3.8 & 78.2$\pm$4.1 \\
    nnPU       & 91.4$\pm$0.6  & 90.2$\pm$0.7  & 84.7$\pm$2.4   & 83.7$\pm$0.6   & 77.1$\pm$4.5 & 80.4$\pm$2.7 \\
    Self-PU    & 90.8$\pm$0.4  & 89.1$\pm$0.7  & 85.1$\pm$0.8   & 83.9$\pm$2.6   & 78.5$\pm$1.1 & 80.8$\pm$2.1 \\
    PAN        & 87.7$\pm$2.4  & 89.9$\pm$3.2  & 87.0$\pm$0.3   & 82.8$\pm$1.0   & 77.7$\pm$2.5 & 79.8$\pm$1.4 \\
    vPU        & 92.6$\pm$1.2  & 90.5$\pm$0.8  & 86.8$\pm$1.2   & 82.5$\pm$1.1   & 78.4$\pm$1.1 & 82.9$\pm$0.7 \\
    MIXPUL     & 90.4$\pm$1.2  & 89.6$\pm$1.2  & 87.0$\pm$1.9   & 87.0$\pm$1.1   & 77.8$\pm$0.7 & 78.9$\pm$1.9 \\
    PULNS      & 91.0$\pm$0.5  & 89.1$\pm$0.8  & 87.2$\pm$0.6   & 83.7$\pm$2.9   & 80.2$\pm$0.8 & 83.6$\pm$0.7 \\
    Dist-PU    & 94.7$\pm$0.4  & 92.4$\pm$0.4  & 86.8$\pm$0.7   & 87.2$\pm$0.9   & 79.8$\pm$0.6 & 82.9$\pm$0.4 \\
    P$^3$MIX-E & 92.6$\pm$0.4  & 91.8$\pm$0.2  & 88.2$\pm$0.4   & 84.7$\pm$0.5   & 80.2$\pm$0.9 & 83.7$\pm$0.7 \\ 
    P$^3$MIX-C & 92.8$\pm$0.6  & 90.4$\pm$0.1  & 88.7$\pm$0.4   & 87.9$\pm$0.5   & 80.7$\pm$0.7 & 84.1$\pm$0.3 \\ \midrule[1pt]
    Ours & \textbf{95.8$\pm$0.3} & \textbf{96.0$\pm$0.3} & \textbf{91.1$\pm$0.2} & \textbf{90.3$\pm$0.1}& \textbf{83.7$\pm$0.3} & \textbf{85.3$\pm$0.6}         \\ \bottomrule[1.5pt]
    \end{tabular}
}

\label{acc_generic}
\end{table*}

\subsubsection{Sythetic datasets.} Our proposed method consistently outperforms all PUL baselines by 1\% to 4\% on all generic benchmark datasets and settings, as shown in Table \ref{acc_generic}, demonstrating its superior performance. Furthermore, many existing PUL methods rely on a given positive prior or make various assumptions that are not available in real-world settings, whereas our method does not require any of them. To avoid inherent challenges such as accumulated errors and estimation bias, we transform the above challenges into a much simpler task of discerning the trend of the model-predicting scores. Considering we can achieve outstanding classification accuracy in unlabeled data, it is natural to expect our method to outperform existing PUL methods. While using some tricks for label noise learning like Co-teaching\cite{han2018co} and large loss criterion\cite{huang2019o2u} could possibly further improve the performance of our method, we believe that in most scenarios, our method can effectively solve existing PUL problems with simplicity.

\begin{table*}[htbp]
\caption{Comparative results($\%$) on Credit Card Fraud dataset (mean±std).} 
\centering
\setlength\tabcolsep{15pt}
\resizebox{0.95\textwidth}{!}
{  
    \begin{tabular}{cccccc}
    \toprule[1.5pt]
    Algorithm  & F1 score               & Recall               & Accuracy               & Precision  & AUC              \\ \midrule[1pt]
    uPU        & 89.5$\pm$3.1           & 83.4$\pm$1.3         & 97.0$\pm$0.2           & 96.5$\pm$3.6     &93.4$\pm$3.1        \\
    nnPU       & 89.9$\pm$1.0           & 83.4$\pm$1.3         & 98.4$\pm$0.1           & 97.4$\pm$1.1     &94.2$\pm$0.9        \\
    nnPU+mixup & 89.0$\pm$2.8           & 82.9$\pm$1.6         & 98.1$\pm$0.1           & 96.0$\pm$3.2     &93.8$\pm$2.9        \\
    Self-PU    & 89.0$\pm$2.4           & 85.8$\pm$2.0         & 99.2$\pm$0.1           & 92.4$\pm$3.4     &95.6$\pm$2.8       \\
    PAN        & 91.5$\pm$0.9           & 85.4$\pm$1.3         & \textbf{99.1$\pm$0.1}  & 98.5$\pm$1.0     &96.6$\pm$1.1        \\
    VPU        & 91.7$\pm$3.9           & 84.9$\pm$5.7         & 98.6$\pm$0.5           & \textbf{99.7$\pm$0.6}  &96.9$\pm$3.1  \\
    MIXPUL     & 82.9$\pm$2.8           & 86.6$\pm$1.3         & 98.4$\pm$0.3           & 79.2$\pm$3.5       &91.3$\pm$0.7      \\
    PULNS      & 89.0$\pm$2.0           & 83.2$\pm$2.1         & 99.0$\pm$0.1           & 95.6$\pm$1.9       &94.5$\pm$0.7      \\
    Dist-PU    & 87.9$\pm$3.4           & 80.2$\pm$4.1         & 98.8$\pm$0.4           & 97.2$\pm$1.6       &96.5$\pm$2.7      \\
    P$^3$MIX-E & 91.9$\pm$2.1           & 87.7$\pm$2.0         & 99.0$\pm$0.1           & 96.5$\pm$1.8      &97.5$\pm$0.9       \\
    P$^3$MIX-C & 90.2$\pm$1.4           & 86.5$\pm$1.8         & 98.8$\pm$0.1           & 94.1$\pm$1.2      &97.3$\pm$1.2       \\ \midrule[1pt]
    Our Method &\textbf{99.1$\pm$0.2}   & \textbf{99.0$\pm$0.2}& \textbf{99.1$\pm$0.1}  & 99.3$\pm$0.1     &\textbf{99.7$\pm$0.1 }        \\ \bottomrule[1.5pt]
    \end{tabular}
}

\label{tab_fraud}
\end{table*}

\begin{table*}[htbp]
\caption{Comparative results($\%$) on Alzheimer dataset (mean±std). }
\small
\centering
\setlength\tabcolsep{15pt}
\resizebox{0.95\textwidth}{!}
{

    \begin{tabular}{cccccc}
    \toprule[1.5pt]
    Algorithm  & F1 score              & Recall                     & Accuracy               & Precision             & AUC                   \\ \midrule[1pt]
    uPU        & 67.6$\pm$2.8          & 66.1$\pm$6.1               & 68.5$\pm$2.2           & 69.7$\pm$3.5          & 73.8$\pm$2.9          \\
    nnPU       & 68.6$\pm$3.2          & 69.5$\pm$7.2               & 68.3$\pm$2.1           & 68.0$\pm$2.3          & 72.9$\pm$2.8          \\
    RP         & 62.1$\pm$5.6          & 64.6$\pm$15.9              & 61.6$\pm$3.2           & 61.9$\pm$4.5          & 66.1$\pm$3.3          \\
    PUSB       & 69.2$\pm$2.4          & 69.3$\pm$2.4               & 69.2$\pm$2.4           & 69.2$\pm$2.4          & 74.4$\pm$2.4          \\
    PUbN       & 70.4$\pm$3.2          & 72.0$\pm$8.4               & 70.0$\pm$1.3           & 69.4$\pm$2.5          & 70.0$\pm$1.3          \\
    Self-PU    & 72.1$\pm$1.1          & 75.4$\pm$5.1               & 70.9$\pm$0.7           & 69.3$\pm$2.5          & 75.9$\pm$1.8          \\
    aPU        & 70.5$\pm$3.4          & 75.7$\pm$8.2               & 68.5$\pm$1.8           & 66.2$\pm$0.9          & 70.7$\pm$3.7          \\
    VPU        & 70.2$\pm$1.1          & 76.7$\pm$3.6               & 67.4$\pm$0.7           & 64.7$\pm$1.1          & 73.1$\pm$0.9          \\
    ImbPU      & 68.8$\pm$1.9          & 70.6$\pm$6.5               & 68.2$\pm$0.8           & 67.5$\pm$2.5          & 73.8$\pm$0.7          \\
    Dist-PU    & 73.7$\pm$1.6          & \textbf{80.1$\pm$5.1}      & 71.6$\pm$0.6           & 68.5$\pm$1.2          & \textbf{77.1$\pm$0.7} \\ \midrule[1pt]
    Our Method & \textbf{74.5$\pm$2.4} & 79.5$\pm$5.8               & \textbf{72.8$\pm$0.9}  & \textbf{70.2$\pm$1.6} & \textbf{77.1$\pm$2.3} \\ \bottomrule[1.5pt]
    \end{tabular}
}
\label{tab_alz}
\vspace{-10pt}
\end{table*}
\subsubsection{Real-world datasets.}  This subsection presents experimental results on two real-world datasets, including one highly imbalanced Credit Fraud dataset. In fraud detection, recall is typically more important than precision or accuracy, as the consequences of missing a fraudulent transaction can be much more severe than flagging a legitimate transaction as fraudulent. As shown in Table \ref{tab_fraud}, our proposed method achieves significantly higher recall rates and F1 scores, as well as comparable accuracy and precision, indicating its ability to better handle highly imbalanced scenarios. Our approach offers a novel perspective compared to traditional prediction-based methods, as the model's predictive trends are not affected by the positive prior, as long as the observation outlined in Section \ref{sectrend} holds. Furthermore, our method also demonstrates comparable performance on the Alzheimer dataset to the state-of-the-art method DistPU, which employs various regularization techniques and data augmentation strategies. In both two real-world settings, our method achieves a balanced good performance on all evaluation metrics which further illustrates its effectiveness.
\begin{figure}[htbp]
  \centering
  \begin{subfigure}[b]{0.49\textwidth}
    \centering
    \includegraphics[width=\textwidth]{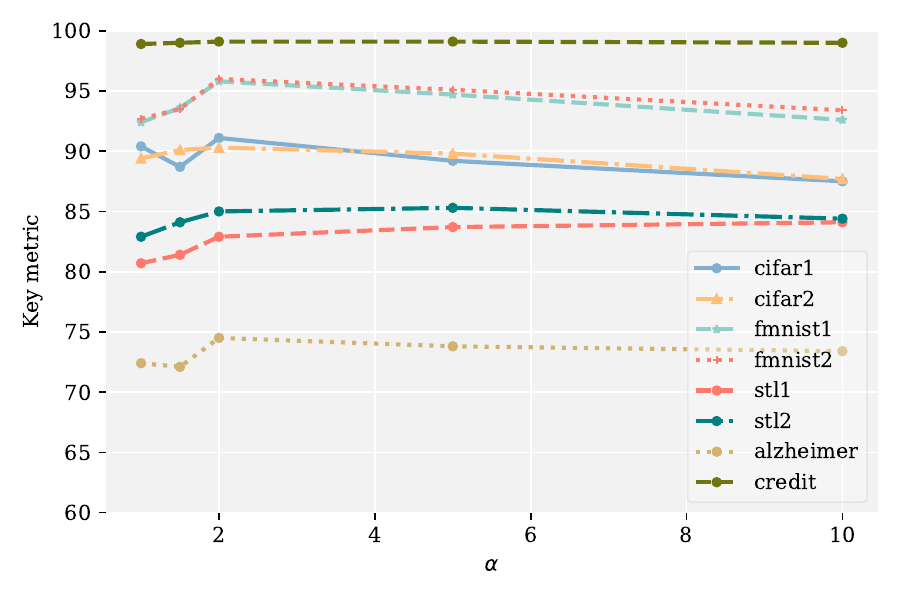}
  \end{subfigure}
  \hfill
  \begin{subfigure}[b]{0.49\textwidth}
    \centering
    \includegraphics[width=\textwidth]{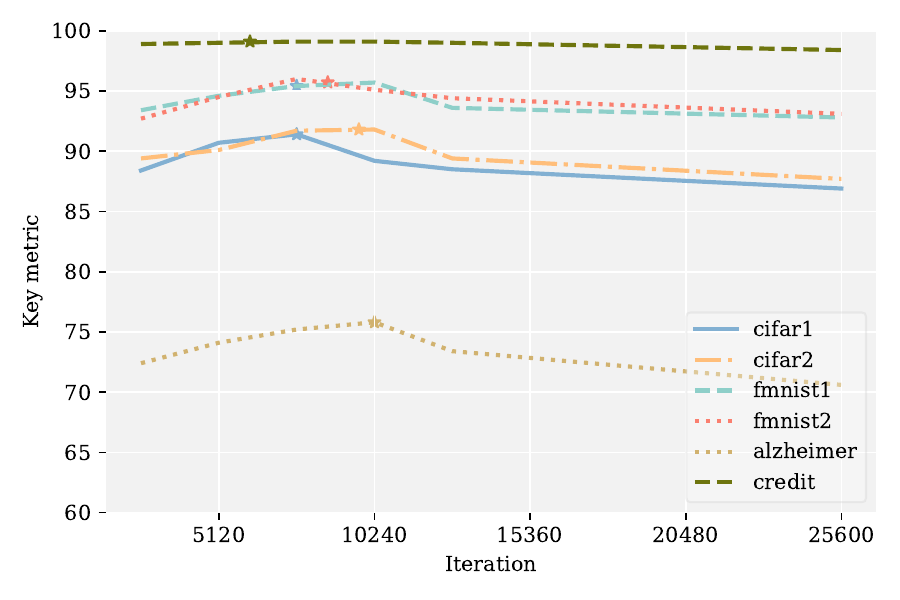}
  \end{subfigure}
  \hfill
\caption{Sensitivity analysis was performed on two parameters: $\alpha$ (left) and stopping iteration (right). The stopping iteration of LZO (also the one we use) is denoted by '$*$' on the right.}
\label{sensitivity}
\end{figure}
\begin{table}[t] 
\caption{Ablation results ($\%$) on CIFAR-10 (acc), Credit Fraud (recall) and Alzheimer (f1 score). "\checkmark" indicates the enabling of the corresponding components.}
\centering
{
    \resizebox{0.95\textwidth}{!}
    {
        \begin{tabular}{c|ccc|cc|ccc}
        \toprule[1.5pt]
        &\multicolumn{3}{c}{\textbf{Trend Measure}} &\multicolumn{2}{c}{\textbf{Clustering}} &\multicolumn{3}{c}{\textbf{Dataset}} \\ 
            \cmidrule{1-4} \cmidrule{5-6} \cmidrule{7-9}
        \textbf{Resampling} & TS   & Simplified TS  & MK & Natural break & k-means 
        & CIFAR10-1 & Credit Fraud & Alzheimer \\ \midrule[1pt]
                   & \checkmark & & &\checkmark &  &84.1  &88.6 &69.2\\
        \checkmark & \checkmark & & & &\checkmark  &89.4  &99.3 &70.5\\
        \checkmark &  & &\checkmark &\checkmark &  &90.2  &99.0 &69.7\\
        \checkmark & & \checkmark& &\checkmark &   &90.7  &99.2 &73.9\\
        \checkmark & \checkmark & & &\checkmark &  &91.1  &99.1 &74.5\\
        \bottomrule[1.5pt]
        \end{tabular}
    }
}
\label{tab_ablation}
\vspace{-8pt}
\end{table}
\subsubsection{Ablation Study.} To investigate the specific effects of different components (Resampling, \textbf{trend score}, and Fisher Natural Break Partition) in our method, we conducted a series of ablation studies and compared them with some popular alternatives. From Table \ref{tab_ablation}, we can draw several observations: (1) The resampling strategy plays a crucial role in our method as it maximizes the discrepancy of the trends in different classes of examples, particularly in the Credit Fraud dataset. It serves as an important factor in amplifying the model's early success, which is the foundation of our further approach towards achieving better performance. (2) Our proposed \textbf{trend score} provides a better evaluation metric than the statistic $\Tilde{S}$ used in the standardized Mann-Kendall test, and the simplified \textbf{trend score} also shows competitive performance. (3) Fisher Natural Break Partition derives deterministic optimal partitions with better statistical properties and empirical performance compared to heuristic k-means. Moreover, it is unrelated to initialization and less time-consuming than the original version, as detailed in Appendix\ref{Natural Breaks Classification}.

\subsubsection{Sensitivity Analysis.}  In this subsection, we investigate the impact of two hyperparameters, namely the scaling parameter $\alpha$ and the stopping iteration (we do not need to manually tune it), on the evaluation of predictive trends for each example. To facilitate comparisons, we set $\alpha$ to $2$ and employ the LZO algorithm \cite{leave} discussed in Section \ref{early stopping} for selecting the stopping epoch in our experiments involving mixed labeled data. As depicted in Figure \ref{sensitivity}, our approach consistently delivers robust outcomes across diverse hyperparameter values. Moreover, the model tends to perform better when $\alpha>1$ and demonstrates basically consistent performance. Figure \ref{sensitivity} confirms the effectiveness of the LZO strategy which is free of manual intervention in the stopping epoch.
\vspace*{-8pt}
\section{Related Works}
\vspace*{-8pt}
For a long time, learning with limited supervision has been a striking task in the machine learning community and PUL is an emerging paradigm of weakly supervised learning \cite{zhouweak, gong2020centroid}. Despite its close relations with some similar concepts, the term PUL is generally accepted from \cite{liupu, denispu,PN}. Currently, the mainstream PUL methods cast this problem as a cost-sensitive classification task through importance reweighting, among which uPU \cite{duupu} is the widely known one. Later, the authors of nnPU \cite{kiryonnpu} suggest that uPU gets overfitting when using flexible and complex models such as Deep Neural Networks and thus propose a non-negative risk estimator. Some recent studies attempt to combine the cost-sensitive method with model's capability to calibrate and distill the labeled set with various techniques like denoise \cite{noisypu}, self-paced curriculum \cite{selfpu} and heuristic mix up \cite{liyour, mixpul}. 

Parallel with the cost-sensitive methods, another branch of PUL methods adopts a heuristic two-step method. The early trials of two-step methods mainly focus on the sample-selection task to form a reliable negative set and further yield the semi-supervised learning framework \cite{twostepbase, outlierdetection, twosteprelabeling, vpu,pgan}. Other two-step methods are mainly derived from the large margin principle to correct the bias caused by unreliable negative data such as Loss Decomposition \cite{shicentroid}, Large margin based calibration and label disambiguation \cite{largecarlibration, largedisambiguation}. Plus, different techniques have been employed to assign labels for unlabeled data in PUL like Graph-based models \cite{graphone,graphtwo}, GAN \cite{genpu, pgan} and Reinforcement learning \cite{pulns} in recent years. Plus, decision tree based PU methods are also investigated in \cite{wiltonpositive}.

Most PUL methods are oriented from a SCAR (selected completely at random) assumption or established on a given class prior. In this respect, there emerges some class prior estimation algorithms specially designed for PUL. PE attempts to minimize the Pearson divergence between the labeled and unlabeled distribution, PEN-L1 \cite{PENL} and MPE \cite{mpe} are then proposed to modify PE by using a simple Best Bin Estimation (BBE) technique. Unfortunately, most class prior estimation algorithms still rely on specific assumptions and the estimates will be unreliable otherwise\cite{menon2015learning}. Regarding the possibility of selection bias in the labeling process, the SCAR assumption is relaxed in \cite{kato2019learning}. VAE-PU is the first generative PUL model without a supposed labeling mechanism like SCAR assumption \cite{vaepu} and further investigated in \cite{asymmetricpu}. For more details about PUL, readers are referred to a recent survey for a comprehensive understanding of this subject \cite{bekker}. 
\vspace*{-8pt}
\section{Conclusion}
\vspace*{-8pt}
This study introduces a novel method for Positive-Unlabeled Learning (PUL) that takes a fresh perspective by identifying the unique characteristics of each example's predictive trend. Our approach is based on two key observations: Firstly, resampling positive examples to create a balanced training distribution can achieve comparable or even superior performance to existing state-of-the-art methods in the early stages of training. Secondly, the predicting scores of negative examples tend to exhibit a consistent decrease, while those of positive examples may initially increase before ultimately decreasing or oscillating. These insights lead us to reframe the central challenge of PUL as a task of discerning the trend of the model predicting scores. We also propose a novel labeling approach that uses statistical methods to identify significant partitions, circumventing the need for manual intervention in determining confidence thresholds or selecting ratios. Extensive empirical studies demonstrate the effectiveness of our method and its potential to contribute to related fields, such as learning from noisy labels and semi-supervised learning.

\section{Acknowledgments and Disclosure of Funding}
This work was supported by the Natural Science Foundation of China (NSFC)  (Grant No.62376126), the National Key R\&D Program of China (2022ZD0114801), National Natural Science Foundation of China (61906089), Natural Science Foundation of China (NSFC) (Grant No.62106102),  Natural Science Foundation of Jiangsu Province (BK20210292), Graduate Research and Practical Innovation Program at Nanjing University of Aeronautics and Astronautics (xcxjh20221601).
\newpage
\bibliography{neurips_2023}

\begin{thebibliography}{10}

\bibitem{arpit2017closer}
D.~Arpit, S.~Jastrz{\k{e}}bski, N.~Ballas, D.~Krueger, E.~Bengio, M.~S. Kanwal,
  T.~Maharaj, A.~Fischer, A.~Courville, Y.~Bengio, et~al.
\newblock A closer look at memorization in deep networks.
\newblock In {\em International conference on machine learning}, pages
  233--242. PMLR, 2017.

\bibitem{bekker}
J.~Bekker and J.~Davis.
\newblock Learning from positive and unlabeled data: A survey.
\newblock {\em Machine Learning}, 109(4):719--760, 2020.

\bibitem{catoni2012challenging}
O.~Catoni.
\newblock Challenging the empirical mean and empirical variance: a deviation
  study.
\newblock In {\em Annales de l'IHP Probabilit{\'e}s et statistiques},
  volume~48, pages 1148--1185, 2012.

\bibitem{graphone}
S.~Chaudhari and S.~Shevade.
\newblock Learning from positive and unlabelled examples using maximum margin
  clustering.
\newblock In {\em International Conference on Neural Information Processing},
  pages 465--473. Springer, 2012.

\bibitem{chawla2002smote}
N.~V. Chawla, K.~W. Bowyer, L.~O. Hall, and W.~P. Kegelmeyer.
\newblock Smote: synthetic minority over-sampling technique.
\newblock {\em Journal of artificial intelligence research}, 16:321--357, 2002.

\bibitem{vpu}
H.~Chen, F.~Liu, Y.~Wang, L.~Zhao, and H.~Wu.
\newblock A variational approach for learning from positive and unlabeled data.
\newblock {\em Advances in Neural Information Processing Systems},
  33:14844--14854, 2020.

\bibitem{chen2021generalized}
P.~Chen, X.~Jin, X.~Li, and L.~Xu.
\newblock A generalized catoni’s m-estimator under finite $\alpha$-th moment
  assumption with $\alpha\in$(1, 2).
\newblock {\em Electronic Journal of Statistics}, 15(2):5523--5544, 2021.

\bibitem{selfpu}
X.~Chen, W.~Chen, T.~Chen, Y.~Yuan, C.~Gong, K.~Chen, and Z.~Wang.
\newblock Self-pu: Self boosted and calibrated positive-unlabeled training.
\newblock In {\em International Conference on Machine Learning}, pages
  1510--1519. PMLR, 2020.

\bibitem{PENL}
M.~Christoffel, G.~Niu, and M.~Sugiyama.
\newblock Class-prior estimation for learning from positive and unlabeled data.
\newblock In {\em Asian Conference on Machine Learning}, pages 221--236. PMLR,
  2016.

\bibitem{stl}
A.~Coates, A.~Ng, and H.~Lee.
\newblock An analysis of single-layer networks in unsupervised feature
  learning.
\newblock In {\em Proceedings of the fourteenth international conference on
  artificial intelligence and statistics}, pages 215--223. JMLR Workshop and
  Conference Proceedings, 2011.

\bibitem{coudray2023risk}
O.~Coudray, C.~Keribin, P.~Massart, and P.~Pamphile.
\newblock Risk bounds for positive-unlabeled learning under the selected at
  random assumption.
\newblock {\em Journal of Machine Learning Research}, 24(107):1--31, 2023.

\bibitem{denispu}
F.~Denis, R.~Gilleron, and F.~Letouzey.
\newblock Learning from positive and unlabeled examples.
\newblock {\em Theoretical Computer Science}, 348(1):70--83, 2005.

\bibitem{duupu}
M.~Du~Plessis, G.~Niu, and M.~Sugiyama.
\newblock Convex formulation for learning from positive and unlabeled data.
\newblock In {\em International conference on machine learning}, pages
  1386--1394. PMLR, 2015.

\bibitem{PN}
C.~Elkan and K.~Noto.
\newblock Learning classifiers from only positive and unlabeled data.
\newblock In {\em Proceedings of the 14th ACM SIGKDD international conference
  on Knowledge discovery and data mining}, pages 213--220, 2008.

\bibitem{mpe}
S.~Garg, Y.~Wu, A.~J. Smola, S.~Balakrishnan, and Z.~Lipton.
\newblock Mixture proportion estimation and pu learning: A modern approach.
\newblock {\em Advances in Neural Information Processing Systems},
  34:8532--8544, 2021.

\bibitem{largecarlibration}
C.~Gong, T.~Liu, J.~Yang, and D.~Tao.
\newblock Large-margin label-calibrated support vector machines for positive
  and unlabeled learning.
\newblock {\em IEEE transactions on neural networks and learning systems},
  30(11):3471--3483, 2019.

\bibitem{gong2020centroid}
C.~Gong, J.~Yang, J.~You, and M.~Sugiyama.
\newblock Centroid estimation with guaranteed efficiency: A general framework
  for weakly supervised learning.
\newblock {\em IEEE Transactions on Pattern Analysis and Machine Intelligence},
  44(6):2841--2855, 2020.

\bibitem{guo2020positive}
T.~Guo, C.~Xu, J.~Huang, Y.~Wang, B.~Shi, C.~Xu, and D.~Tao.
\newblock On positive-unlabeled classification in gan.
\newblock In {\em Proceedings of the IEEE/CVF Conference on Computer Vision and
  Pattern Recognition}, pages 8385--8393, 2020.

\bibitem{hamed2008trend}
K.~H. Hamed.
\newblock Trend detection in hydrologic data: the mann--kendall trend test
  under the scaling hypothesis.
\newblock {\em Journal of hydrology}, 349(3-4):350--363, 2008.

\bibitem{hamed1998modified}
K.~H. Hamed and A.~R. Rao.
\newblock A modified mann-kendall trend test for autocorrelated data.
\newblock {\em Journal of hydrology}, 204(1-4):182--196, 1998.

\bibitem{hammoudeh2020learning}
Z.~Hammoudeh and D.~Lowd.
\newblock Learning from positive and unlabeled data with arbitrary positive
  shift.
\newblock {\em Advances in Neural Information Processing Systems},
  33:13088--13099, 2020.

\bibitem{han2018co}
B.~Han, Q.~Yao, X.~Yu, G.~Niu, M.~Xu, W.~Hu, I.~Tsang, and M.~Sugiyama.
\newblock Co-teaching: Robust training of deep neural networks with extremely
  noisy labels.
\newblock {\em Advances in neural information processing systems}, 31, 2018.

\bibitem{twosteprelabeling}
F.~He, T.~Liu, G.~I. Webb, and D.~Tao.
\newblock Instance-dependent pu learning by bayesian optimal relabeling, 2018.

\bibitem{genpu}
M.~Hou, B.~Chaib-Draa, C.~Li, and Q.~Zhao.
\newblock Generative adversarial positive-unlabelled learning, 2017.

\bibitem{hsieh}
C.-J. Hsieh, N.~Natarajan, and I.~Dhillon.
\newblock Pu learning for matrix completion.
\newblock In {\em International conference on machine learning}, pages
  2445--2453. PMLR, 2015.

\bibitem{hsieh2019classification}
Y.-G. Hsieh, G.~Niu, and M.~Sugiyama.
\newblock Classification from positive, unlabeled and biased negative data.
\newblock In {\em International Conference on Machine Learning}, pages
  2820--2829. PMLR, 2019.

\bibitem{pgan}
W.~Hu, R.~Le, B.~Liu, F.~Ji, J.~Ma, D.~Zhao, and R.~Yan.
\newblock Predictive adversarial learning from positive and unlabeled data.
\newblock In {\em Proceedings of the AAAI Conference on Artificial
  Intelligence}, volume~35, pages 7806--7814, 2021.

\bibitem{huang2019o2u}
J.~Huang, L.~Qu, R.~Jia, and B.~Zhao.
\newblock O2u-net: A simple noisy label detection approach for deep neural
  networks.
\newblock In {\em Proceedings of the IEEE/CVF international conference on
  computer vision}, pages 3326--3334, 2019.

\bibitem{dividemix}
L.~Junnan and S.~C. Hoi.
\newblock Dividemix: Learning with noisy labelsas semi-supervised learning.
\newblock In {\em ICLR. International Conference on Learning Representations
  (ICLR)}, 2020.

\bibitem{kato2019learning}
M.~Kato, T.~Teshima, and J.~Honda.
\newblock Learning from positive and unlabeled data with a selection bias.
\newblock In {\em International conference on learning representations}, 2019.

\bibitem{kiryonnpu}
R.~Kiryo, G.~Niu, M.~C. Du~Plessis, and M.~Sugiyama.
\newblock Positive-unlabeled learning with non-negative risk estimator.
\newblock {\em Advances in neural information processing systems}, 30, 2017.

\bibitem{cifar}
A.~Krizhevsky, G.~Hinton, et~al.
\newblock Learning multiple layers of features from tiny images, 2009.

\bibitem{liyour}
C.~Li, X.~Li, L.~Feng, and J.~Ouyang.
\newblock Who is your right mixup partner in positive and unlabeled learning.
\newblock In {\em International Conference on Learning Representations}, 2022.

\bibitem{leave}
W.~Li, C.~Geng, and S.~Chen.
\newblock Leave zero out: Towards a no-cross-validation approach for model
  selection, 2020.

\bibitem{outlierdetection}
X.~Li, B.~Liu, and S.-K. Ng.
\newblock Learning to identify unexpected instances in the test set.
\newblock In {\em IJCAI}, volume~7, pages 2802--2807, 2007.

\bibitem{liupu}
X.-L. Li and B.~Liu.
\newblock Learning from positive and unlabeled examples with different data
  distributions.
\newblock In {\em European conference on machine learning}, pages 218--229.
  Springer, 2005.

\bibitem{early}
S.~Liu, J.~Niles-Weed, N.~Razavian, and C.~Fernandez-Granda.
\newblock Early-learning regularization prevents memorization of noisy labels.
\newblock {\em Advances in neural information processing systems},
  33:20331--20342, 2020.

\bibitem{pulns}
C.~Luo, P.~Zhao, C.~Chen, B.~Qiao, C.~Du, H.~Zhang, W.~Wu, S.~Cai, B.~He,
  S.~Rajmohan, et~al.
\newblock Pulns: Positive-unlabeled learning with effective negative sample
  selector.
\newblock In {\em Proceedings of the AAAI Conference on Artificial
  Intelligence}, volume~35, pages 8784--8792, 2021.

\bibitem{mahsereci2017early}
M.~Mahsereci, L.~Balles, C.~Lassner, and P.~Hennig.
\newblock Early stopping without a validation set.
\newblock {\em arXiv preprint arXiv:1703.09580}, 2017.

\bibitem{menon2015learning}
A.~Menon, B.~Van~Rooyen, C.~S. Ong, and B.~Williamson.
\newblock Learning from corrupted binary labels via class-probability
  estimation.
\newblock In {\em International conference on machine learning}, pages
  125--134. PMLR, 2015.

\bibitem{moore1990uncertainty}
D.~Moore.
\newblock Uncertainty. on the shoulders of giants: new approaches to numeracy.
  la steen, 1990.

\bibitem{vaepu}
B.~Na, H.~Kim, K.~Song, W.~Joo, Y.-Y. Kim, and I.-C. Moon.
\newblock Deep generative positive-unlabeled learning under selection bias.
\newblock In {\em Proceedings of the 29th ACM International Conference on
  Information \& Knowledge Management}, pages 1155--1164, 2020.

\bibitem{northcutt2017learning}
C.~G. Northcutt, T.~Wu, and I.~L. Chuang.
\newblock Learning with confident examples: Rank pruning for robust
  classification with noisy labels.
\newblock In {\em Conference on Uncertainty in Artificial Intelligence}, 2017.

\bibitem{km}
H.~Ramaswamy, C.~Scott, and A.~Tewari.
\newblock Mixture proportion estimation via kernel embeddings of distributions.
\newblock In {\em International conference on machine learning}, pages
  2052--2060. PMLR, 2016.

\bibitem{ren}
Y.~Ren, D.~Ji, and H.~Zhang.
\newblock Positive unlabeled learning for deceptive reviews detection.
\newblock In {\em Proceedings of the 2014 conference on empirical methods in
  natural language processing (EMNLP)}, pages 488--498, 2014.

\bibitem{shicentroid}
H.~Shi, S.~Pan, J.~Yang, and C.~Gong.
\newblock Positive and unlabeled learning via loss decomposition and centroid
  estimation.
\newblock In {\em IJCAI}, pages 2689--2695, 2018.

\bibitem{fixmatch}
K.~Sohn, D.~Berthelot, N.~Carlini, Z.~Zhang, H.~Zhang, C.~A. Raffel, E.~D.
  Cubuk, A.~Kurakin, and C.-L. Li.
\newblock Fixmatch: Simplifying semi-supervised learning with consistency and
  confidence.
\newblock {\em Advances in Neural Information Processing Systems}, 33:596--608,
  2020.

\bibitem{imbpu}
G.~Su, W.~Chen, and M.~Xu.
\newblock Positive-unlabeled learning from imbalanced data.
\newblock In {\em IJCAI}, pages 2995--3001, 2021.

\bibitem{noisypu}
D.~Tanaka, D.~Ikami, and K.~Aizawa.
\newblock A novel perspective for positive-unlabeled learning via noisy labels,
  2021.

\bibitem{vapnik1999nature}
V.~Vapnik.
\newblock {\em The nature of statistical learning theory}.
\newblock Springer science \& business media, 1999.

\bibitem{asymmetricpu}
C.~Wang, J.~Pu, Z.~Xu, and J.~Zhang.
\newblock Asymmetric loss for positive-unlabeled learning.
\newblock In {\em 2021 IEEE International Conference on Multimedia and Expo
  (ICME)}, pages 1--6. IEEE, 2021.

\bibitem{mixpul}
T.~Wei, F.~Shi, H.~Wang, W.-W.~T. Li, et~al.
\newblock Mixpul: Consistency-based augmentation for positive and unlabeled
  learning, 2020.

\bibitem{wiltonpositive}
J.~Wilton, A.~Koay, R.~Ko, M.~Xu, and N.~Ye.
\newblock Positive-unlabeled learning using random forests via recursive greedy
  risk minimization.
\newblock In {\em Advances in Neural Information Processing Systems}, 2022.

\bibitem{xia2021sample}
X.~Xia, T.~Liu, B.~Han, M.~Gong, J.~Yu, G.~Niu, and M.~Sugiyama.
\newblock Sample selection with uncertainty of losses for learning with noisy
  labels.
\newblock {\em arXiv preprint arXiv:2106.00445}, 2021.

\bibitem{xiaofashion}
H.~Xiao, K.~Rasul, and R.~Vollgraf.
\newblock Fashion-mnist: a novel image dataset for benchmarking machine
  learning algorithms, 2017.

\bibitem{yang}
P.~Yang, X.-L. Li, J.-P. Mei, C.-K. Kwoh, and S.-K. Ng.
\newblock Positive-unlabeled learning for disease gene identification.
\newblock {\em Bioinformatics}, 28(20):2640--2647, 2012.

\bibitem{twostepbase}
H.~Yu, J.~Han, and K.~C.-C. Chang.
\newblock Pebl: positive example based learning for web page classification
  using svm.
\newblock In {\em Proceedings of the eighth ACM SIGKDD international conference
  on Knowledge discovery and data mining}, pages 239--248, 2002.

\bibitem{flexmatch}
B.~Zhang, Y.~Wang, W.~Hou, H.~Wu, J.~Wang, M.~Okumura, and T.~Shinozaki.
\newblock Flexmatch: Boosting semi-supervised learning with curriculum pseudo
  labeling.
\newblock {\em Advances in Neural Information Processing Systems}, 34, 2021.

\bibitem{zhang2021understanding}
C.~Zhang, S.~Bengio, M.~Hardt, B.~Recht, and O.~Vinyals.
\newblock Understanding deep learning (still) requires rethinking
  generalization.
\newblock {\em Communications of the ACM}, 64(3):107--115, 2021.

\bibitem{largedisambiguation}
C.~Zhang, D.~Ren, T.~Liu, J.~Yang, and C.~Gong.
\newblock Positive and unlabeled learning with label disambiguation.
\newblock In {\em IJCAI}, pages 4250--4256, 2019.

\bibitem{zhaodist}
Y.~Zhao, Q.~Xu, Y.~Jiang, P.~Wen, and Q.~Huang.
\newblock Dist-pu: Positive-unlabeled learning from a label distribution
  perspective.
\newblock In {\em Proceedings of the IEEE/CVF Conference on Computer Vision and
  Pattern Recognition}, pages 14461--14470, 2022.

\bibitem{graphtwo}
D.~Zhou, O.~Bousquet, T.~Lal, J.~Weston, and B.~Scholkopf.
\newblock Learning with local and global consistency.
\newblock {\em Advances in neural information processing systems}, 16, 2003.

\bibitem{zhou2021step}
Z.~Zhou, L.-Z. Guo, Z.~Cheng, Y.-F. Li, and S.~Pu.
\newblock Step: Out-of-distribution detection in the presence of limited
  in-distribution labeled data.
\newblock {\em Advances in Neural Information Processing Systems},
  34:29168--29180, 2021.

\bibitem{zhouweak}
Z.-H. Zhou.
\newblock A brief introduction to weakly supervised learning.
\newblock {\em National science review}, 5(1):44--53, 2018.

\end{thebibliography}

\newpage
\appendix
\newtheorem{lemma}{Lemma}[section]
\section{Analysis for Resampling Method}
\label{resampling model}
\begin{figure}[ht]
    \centering
    \includegraphics[scale=0.38]{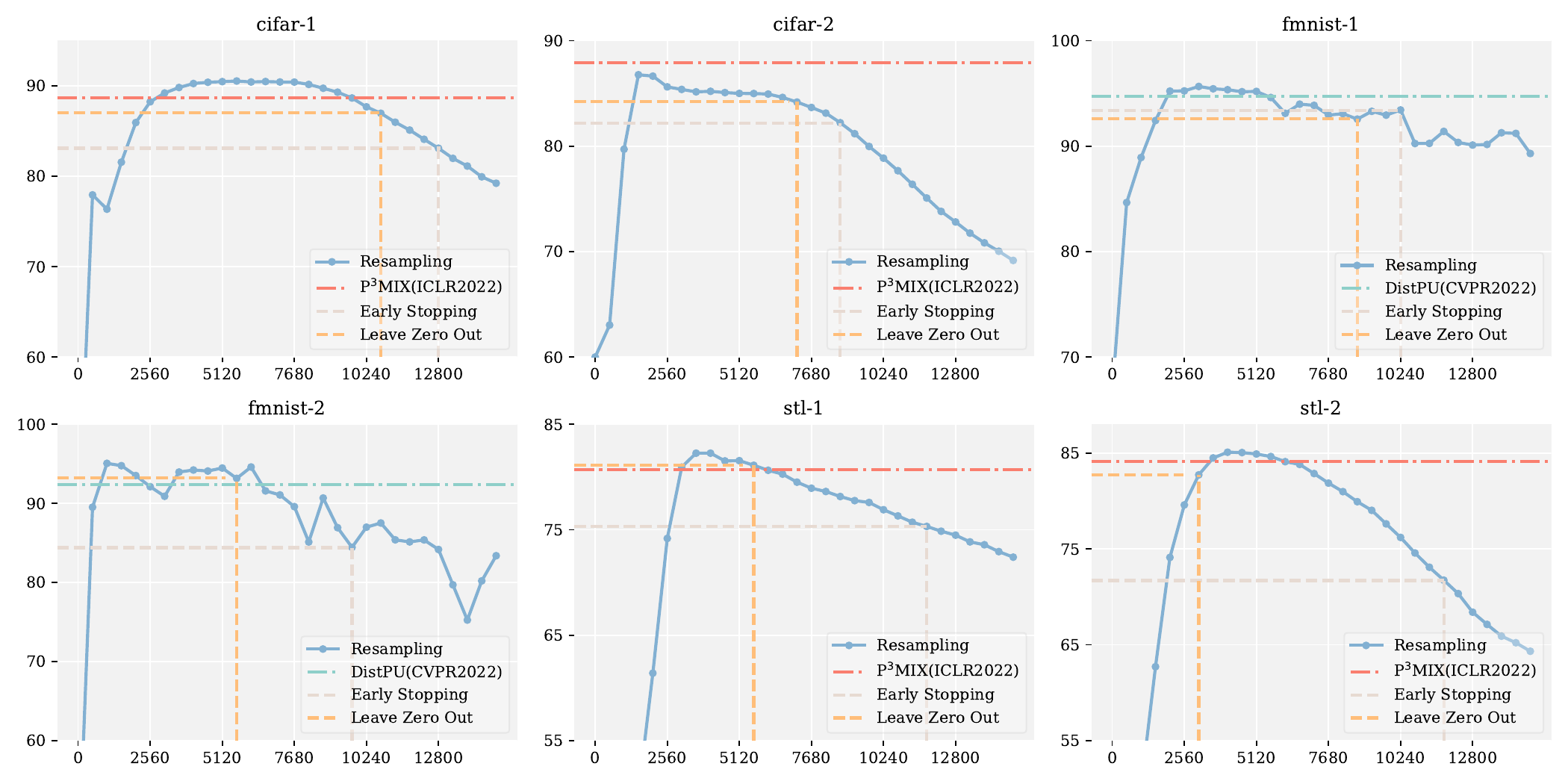} 
    \caption{The accuracy of our resampling method on various settings across all 3 generic datasets. The horizontal line represents the accuracy of the state-of-the-art methods.}
    \label{test_acc}
\end{figure}
\textbf{Empirical Results.} Specifically, we employ the negativity assumption and resample positive data to achieve a balanced training distribution. Despite its simplicity, such a resampling approach achieves great empirical success as shown in Figure\ref{test_acc}, as it highlights the value of precious labels and mitigates the negative impact brought by false negatives and imbalanced label distribution. The outcomes suggest that the early predictive ability of the model could potentially facilitate our efforts in classification tasks. However, determining the optimal epoch to stop training and select the best model still remains a challenging task in PUL due to the absence of a precise validation set. For early stopping, we follow the settings in \cite{mixpul} and hold out 500 positive examples as a validation set. For LZO, we use an augmented validation set based on mix-up techniques following\cite{leave}.

\newtheorem{assumption}{Assumption}[section]
\begin{assumption}
\label{ass1} 
We consider a naive situation where positive and negative data are drawn from a mixture of two Gaussians in $\mathbb{R}^p$ respectively and the dataset consists of $n$ i.i.d. samples from the following distributions:
\begin{equation}
\begin{aligned}
    &\mathbb{P}(x|y=0)\sim \mathcal{N}(+v, \sigma^2I_{p\times p}),\\
    &\mathbb{P}(x|y=1)\sim \mathcal{N}(-v, \sigma^2I_{p\times p}).
\end{aligned}
\end{equation}
where $v$ is an arbitrary unit vector in $\mathbb{R}^p$ and $\sigma^2$ is a small constant. Please keep in mind that the clusters are two spheres with radii $\sigma \sqrt{p}>>2$ when $n,p\to \infty$ which makes this classification nontrivial. This binary classifier is trained by simply discriminating between positive and unlabeled data (i.i.d. sampled from the true distribution). 
\begin{equation}
     \mathbb{P}(x_u)\sim \pi \mathbb{P}(x|y=0) +(1-\pi)\mathbb{P}(x|y=1).
\end{equation}
\end{assumption}

\subsection{Bayesian Decision Hyperplane}
\label{decision hyperplane}

\newtheorem{proposition}{Proposition}[section]
\begin{proposition}
Under Assumption \ref{ass1}, the Bayesian optimal decision hyperplane $h_{pu}$ derived from the model using resampling strategy under a PU setting is equivalent to the Bayesian optimal decision hyperplane $h_{pn}^*$ under a balanced PN binary classification setting.
\begin{equation}
    h_{pu} = h_{pn}^*.
\end{equation}
\end{proposition}

\begin{proof}
We first discuss the decision hyperplane when both positive and negative data are available. By the virtue of Bayes' theorem, the score function  $g_{pn}$ and decision hyperplane $h_{pn}$ separating each category at the same probability should be formulated as:
\begin{equation}
    \begin{aligned}
         g_{pn}(x)&=g_{p}(x)-g_{n}(x)\\
                  &=ln[\mathbb{P}(x_p)\mathbb{P}(y=0)]-ln[\mathbb{P}(x_n)\mathbb{P}(y=1)]\\
                  &=ln\frac{\mathbb{P}(x|y=0)}{\mathbb{P}(x|y=1)} + ln\frac{\pi}{1-\pi}\\
                  &=ln\frac{ N(+v, \sigma^2I_{p\times p})}{ N(-v, \sigma^2I_{p\times p})} + ln\frac{\pi}{1-\pi}\\
                  &=\frac{2v^tx}{\sigma^2}+ ln\frac{\pi}{1-\pi}.
    \end{aligned}
\end{equation}
\vspace{-0.05cm}
\begin{equation}
    g_{pn}(x)=0\Rightarrow h_{pn}:2v^tx+\sigma^2ln\frac{\pi}{1-\pi}=0.
    \label{pn}
\end{equation}

There exists an ideal decision hyperplane $h^{*}_{pn}$ when the positive and negative data is balanced distributed ( $\pi=1-\pi=0.5$). 
\begin{equation}
    h^{*}_{pn}(x):2v^tx=0.
    \label{ideal}
\end{equation}
When the distribution of negative data is unknown to us, we simply take the negativity assumption to make the classification by differentiating unlabeled data and positive data. Thus, the score function  $g_{pu}$ and decision hyperplane $h_{pu}$ can be formulated as:
\begin{equation}
    \begin{aligned}
        g_{pu}(x)   &=g_{p}(x)-g_{u}(x)\\
                    &=ln[\mathbb{P}(x_p)\mathbb{P}(l)]-ln[\mathbb{P}(x_u)\mathbb{P}(u)]\\
                    &=ln\frac{\mathbb{P}(x|y=0)}{\pi \mathbb{P}(x|y=0) +(1-\pi)\mathbb{P}(x|y=1)} + ln\frac{\mathbb{P}(l)}{\mathbb{P}(u)}\\
                    &=ln\ N(+v, \sigma^2I_{p\times p})+ ln\frac{|\mathcal{P}|}{|\mathcal{U}|}-ln[\pi N(+v, \sigma^2I_{p\times p})+(1-\pi)N(-v, \sigma^2I_{p\times p}) \\
                    &=-ln[(1-\pi)exp({\frac{-2v^tx}{\sigma^2}})+\pi] + ln\frac{|\mathcal{P}|}{|\mathcal{U}|}.
    \end{aligned}
\end{equation}
\vspace{-0.05cm}
\begin{equation}
    g_{pu}(x)=0 \Rightarrow h_{pu}:2v^tx+\sigma^2(ln(\frac{|\mathcal{P}|}{|\mathcal{U}|}-\pi)-ln(1-\pi))=0.
    \label{pu}
\end{equation}
When adopting a resampling strategy, the $|\mathcal{P}|/|\mathcal{U}|$ is set to 1, $h_{pu} = h_{pn}^*$.
\end{proof}
We can also observe that when $\pi=0$, Eq.\ref{pu} degrades to Eq.\ref{pn}, which corresponds to the special case where the unlabeled set consists only of negative examples. However, it should be noted that in most cases, $|\mathcal{P}|/|\mathcal{U}|$ is less than $\pi$, making the decision hyperplane unlearnable. This underscores that label noise and data imbalance, introduced by the negativity assumption, are two key reasons for model degradation during the latter training phase. Therefore, we can consider $|\mathcal{P}|/|\mathcal{U}|$ as a flexible coefficient that controls the relative importance of data belonging to different classes. When we adopt a resampling strategy like our baseline, we aim to set this coefficient to 1, enabling us to derive an optimal decision hyperplane as shown in Eq.\ref{ideal}.

\subsection{Early Learning Phenomenon in PU Setting}
To better illustrate model's early success when adoping the resampling strategy, we reformalize the theorem of Early Learning phenomenon given by \cite{early} in a linear model and verify that this phenomenon also exists in PUL when taking cross entropy(CE) loss as the loss function. We first show that, for the first T iterations, the negative gradient has a constant correlation with $v$. (Note that, by contrast, a random vector in $\mathbb{R}^p$ typically has a negligible correlation with $v$.) Afterward, the false pseudo labels given by negativity assumption are memorized asymptotically.

\begin{lemma}
    Under Assumption \ref{ass1}, denote by $\{S_t\}$ the iterates of gradient descent with step size $\eta$. For any $c \in(0,1)$, there exists a constant $\sigma_c$ such that, if $\sigma\le \sigma_c$ and $p/n \in (1-c/2,1)$, then with probability $1-o(1)$ as $n$, $p\to \infty$ there exists a $T=\Omega(1/\eta)$ such that:
    \begin{itemize}
        \item \textbf{Early learning succeeds:} For $t<T$, $-\nabla \mathcal{L}_{CE}(S_t)$ is well correlated with the correct separator $v$, and at $t=T$ the classifier has higher accuracy on the wrongly labeled examples than at initialization.
    \end{itemize}
    \begin{itemize}
        \item \textbf{Memorization occurs:} As $t \to \infty$, the classifier $S_t$ memorizes all noisy labels.
    \end{itemize}
\end{lemma}

The only specialness of PUL setting is that the ratio of noise is given by negativity assumption controlled by the positive prior $\pi$. However, this only affects the constant $c$ and corresponding $\sigma_c$. Readers are referred to \cite{early} for detailed proof.

Combining the above empirical results and theoretical explanation, we better understand the capacity of resampling methods in the early stage of training.

\subsection{Threshold Selection}
\begin{figure}[ht]
  \centering
  \begin{subfigure}[b]{0.49\textwidth}
    \centering
    \includegraphics[width=\textwidth]{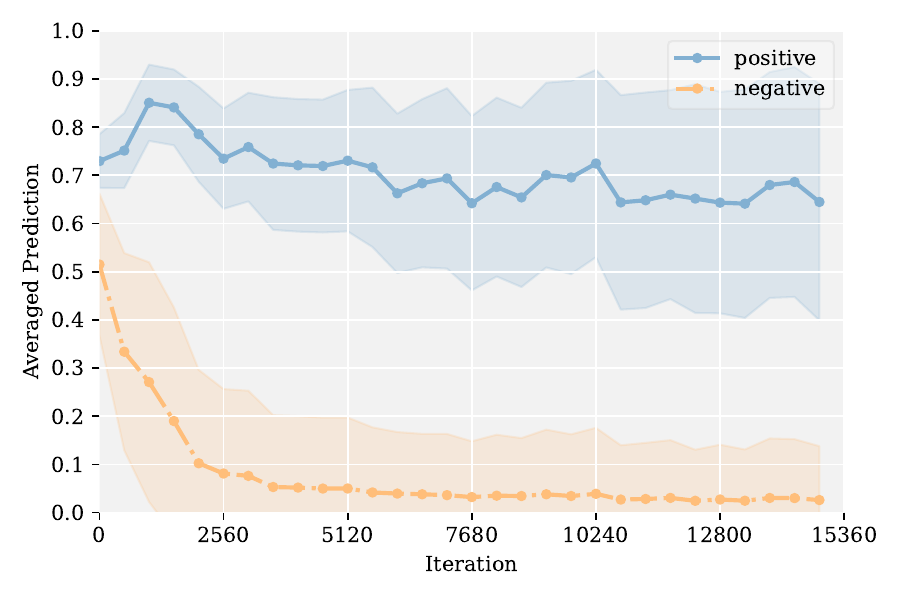}
    \label{uncertain_fmnist1}
  \end{subfigure}
  \hfill
  \begin{subfigure}[b]{0.49\textwidth}
    \centering
    \includegraphics[width=\textwidth]{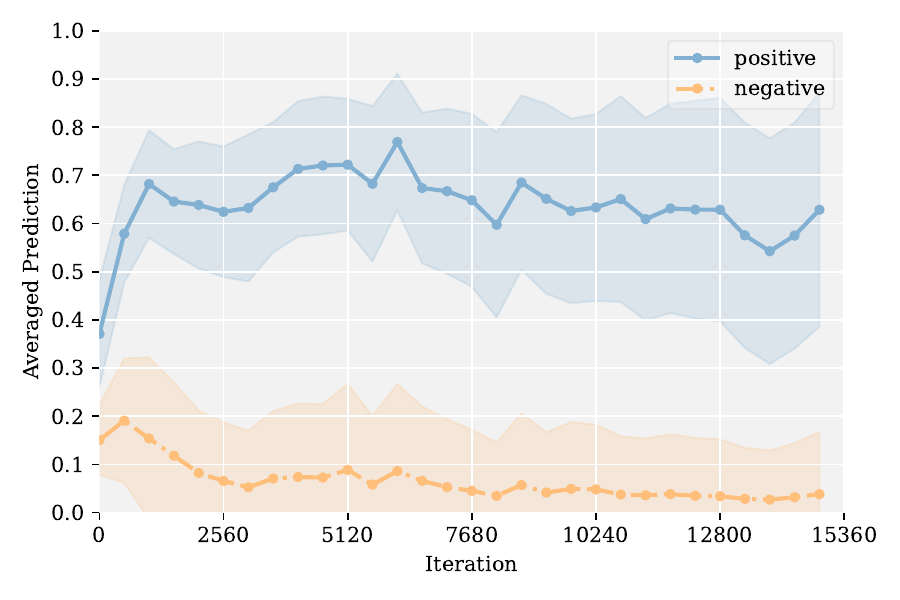}
    \label{uncertain_fmnist2}
  \end{subfigure}
  \hfill
  \begin{subfigure}[b]{0.49\textwidth}
    \centering
    \includegraphics[width=\textwidth]{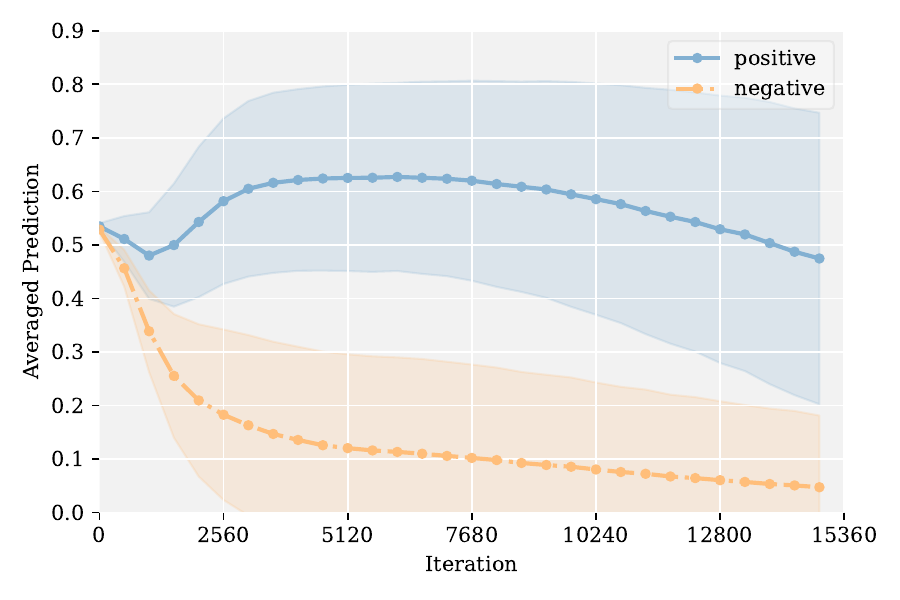}
    \label{uncertain_cifar1_appendix}
  \end{subfigure}
  \hfill
  \begin{subfigure}[b]{0.49\textwidth}
    \centering
    \includegraphics[width=\textwidth]{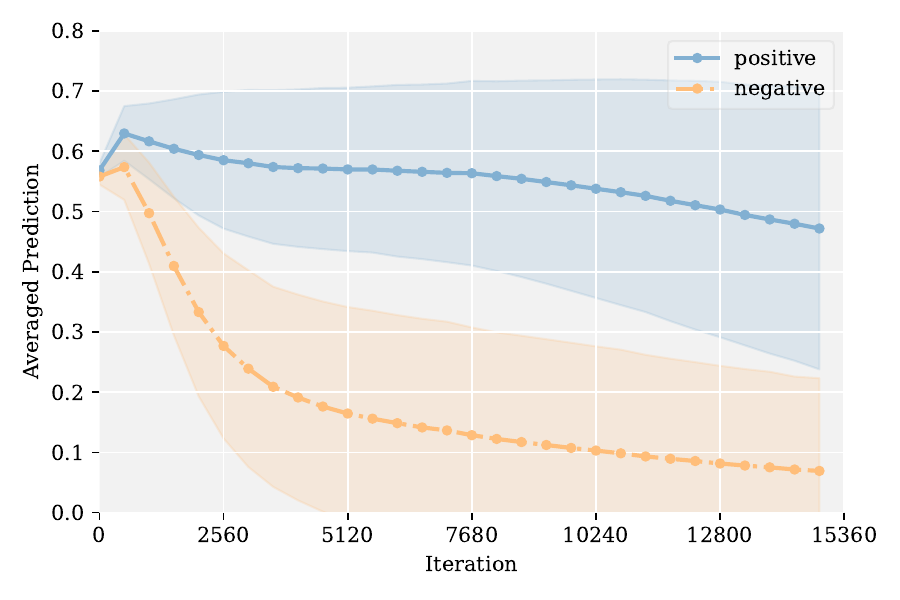}
    \label{uncertain_cifar2_appendix}
  \end{subfigure}
  \hfill
\caption{Averaged prediction confidence with a standard deviation of positive and negative examples on FMNIST1 (upper left), FMNIST2 (upper right), CIFAR10-1 (lower left) and CIFAR10-2 (lower right). }
\label{uncertain_appendix}
\end{figure}

In this section, we present additional predictions and standard deviations obtained from four different settings utilizing CIFAR10 and FMNIST datasets. Notably, as illustrated in Figure \ref{uncertain_appendix}, mislabeling errors of positive examples in the unlabeled set as negatives tend to increase with continued training when the threshold is set at 0.5. These results underscore the importance and challenge of accurately distinguishing between positive and negative examples in PUL tasks. Moreover, our findings indicate that differences between positive and negative examples are reflected in both the predictive trends and magnitudes of model-predicted scores. It also can be seen that, as the training progresses, the interval for an appropriate threshold shrinks.

\section{Mann-Kendall Test}
\label{MK test}
The Mann-Kendall test is a non-parametric test used to determine if a time series has a trend over time. The test calculates the Mann-Kendall statistic $S$ and the variance $Var(S)$. The test is performed by calculating:

\begin{equation}
S = \sum_{i=1}^{n-1} \sum_{j=i+1}^{n} sign(x_j - x_i).
\end{equation}

where $x$ is the time series data, n is the number of observations, and $sign()$ is the sign function that returns $-1$ if its argument is negative, $0$ if its argument is zero, and $1$ if its argument is positive. The variance of $S$ is calculated as:

\begin{equation}
Var(S) = \frac{n(n-1)(2n+5)-\sum_{p=1}^{g} t_p(t_p-1)(2t_p+5)}{18}.
\end{equation}

where $g$ is the number of tied groups, $t_p$ is the number of tied values in the $p$th group. If the absolute value of $S$ is greater than the critical value $(\alpha/2)$ times the standard error of $SE(S)$, where $\alpha$ is the significance level, then the null hypothesis of no trend is rejected. The standard error of $S$ is calculated as:

\begin{equation}
    Z_{MK}=
    \begin{cases}
        \frac{S-1}{\sqrt{VAR(S)}}, S>0 \\
        \frac{S}{\sqrt{VAR(S)}}, S=0\\
        \frac{S+1}{\sqrt{VAR(S)}}, S<0\\
    \end{cases}
\end{equation}

To compute the significance of the Mann-Kendall test, we compare the absolute value of the Mann-Kendall statistic ($Z_{MK}$) to the critical value ($Z_{1-\alpha/2}$). The critical value depends on the level of significance ($\alpha$) chosen and can be obtained from statistical tables or calculated using the software. If $|Z_{MK}| > Z_{1-\alpha/2}$, then the null hypothesis of no trend is rejected and we conclude that there is a significant trend present in the data.

The $\gamma$-value can also be calculated to determine the level of significance of the test. The $\gamma$-value is the probability of observing a Mann-Kendall statistic as extreme or more extreme than the observed value under the null hypothesis of no trend. If the $\gamma$-value is less than the chosen level of significance ($\alpha$), then we reject the null hypothesis and conclude that there is a significant trend (either increasing or decreasing) in the data. If $\gamma$ is bigger than $\alpha$, we conclude there is no trend in this time series data.

To compute the $\gamma$-value, we first calculate the standardized test statistic ($Z$). Then, we calculate the probability of observing a $Z$ value as extreme or more extreme than the observed value using a normal distribution table or software. The $\gamma$-value can be obtained by using the z-table.

\section{Proof of Theorem}
\label{proof}

\label{lemma1}
\begin{lemma}
    \textbf{$C_r$-inequlity:} For any $a, b\in \mathbb{R}$ and $p > 0$, we have:
    \begin{equation}
        {|a+b|}^p \leq max\{2^{p-1},1\}(|a|^p+|b|^p),
    \end{equation}
    and if $p > 1$, it is easy to verify:
    \begin{equation}
        {|a+b|}^p \leq 2^{p-1}(|a|^p+|b|^p).
    \end{equation}
\end{lemma}

Before giving detailed proof, we first rewrite it as a reminder:

\begin{theorem}
Let $P=\{p_{ij}|1\leq i\leq t-1, 2\leq j \leq t,i<j\}$ be an observation set of changes in predictions in which $\Tilde{S}$ is the statistic in the standardized Mann-Kendall test and $\sigma^2$ is the variance of $P$. By exploiting the non-decreasing influence function $\psi(x)$, for any $\epsilon>0$, we have the following bound with probability at least $1-2\epsilon$:
\begin{equation}
        |\alpha\Tilde{S}-\hat{S}|<\frac{2\alpha\sigma\sqrt{\frac{2log(\epsilon^{-1})}{t(t-1)}}}{1-\sqrt{\frac{2log(\epsilon^{-1})}{t(t-1)\alpha^2\sigma^2}}}=\textit{O}\big(({log(\epsilon^{-1}}))^\frac{1}{2}t^{-1}\big).
\end{equation}
\end{theorem}
\begin{proof}
    We first specify some notions here for simplicity :
    \begin{equation}
        \alpha\Tilde{S}=\frac{1}{t(t-1)} \sum_{i=1}^{t-1}\sum_{j=i+1}^{t}\alpha\Delta p_{ij},\hspace{0.1cm}\Delta p_{ij}=p_j-p_i,\hspace{0.1cm}\alpha>0.
    \end{equation}
    \begin{equation}
        S=\frac{2}{t(t-1)}\sum_{i=1}^{t-1}\sum_{j=i+1}^{t}\psi(\alpha\Delta p_{ij}),\hspace{0.1cm}\Delta p_{ij}=p_j-p_i,\hspace{0.1cm}\alpha>0.
    \end{equation}
    \vspace{0.1cm}
    \begin{equation}
        \psi(x)=sign(x)\cdot log(1+|x|+x^2/2).
    \end{equation}
    
     As suggested in \cite{catoni2012challenging}, we can assume the upper and lower bounds of the proposed \textbf{trend score} $S$ as $S^-$ and $S^+$:
    \begin{equation}
        S^-\leq S \leq S^+.
    \end{equation}
    Besides, although $\psi$ is not derivative of some explicit error function, we will use it in the same way and consider it as an influence function. For some positive real parameter $\beta$, we will build our estimator $\hat{S}_\beta$ as the solution of the following equation:
    \begin{equation}
        \sum_{i=1}^{t-1}\sum_{j=i+1}^{t} \psi[\beta(\alpha \Delta p_{ij}-\hat{S}_\beta)]=0.
    \end{equation}
    In fact, we choose the widest possible choice of the M estimator to derive a relatively stabilized empirical mean by making the smallest possible change that is closest to the empirical mean. Then we introduce the quantity and the exponential moment inequalities, from which deviation bounds will follow:
    \begin{equation}
        r(S)=\frac{2}{\beta t(t-1)}\sum_{i=1}^{t-1}\sum_{j=i+1}^{t} \psi[\beta(\alpha\Delta p_{ij}-S)],\hspace{0.1cm} S \in \mathbb{R}.
    \end{equation}
    Simply following the assumptions and \textbf{Proposition 2.1} in \cite{catoni2012challenging}, we can derive the following exponential moment inequalities through \textbf{Lemma\ref{lemma1}}:
    \begin{equation}
    \label{eqemi1}
        \begin{aligned}
            \mathbb{E}\big[e^{\frac{\beta t(t-1)r(S)}{2}}\big]
            &=\mathbb{E}[e^{\sum_{i=1}^{t-1}\sum_{j=i+1}^{t} \psi\big(\beta(\alpha\Delta p_{ij}-S)\big)}] \\
            &={\Big(\mathbb{E}\big[e^{\psi\big(\beta(\alpha\Delta p_{ij}-S)\big)}\big]\Big)}^{\frac{t(t-1)}{2}}\\
            &\leq {\Big(\mathbb{E}\big[1+\beta(\alpha\Delta p_{ij}-S)+\frac{\beta^2}{2}{(\alpha\Delta p_{ij}-S)}^2\big]\Big)}^{\frac{t(t-1)}{2}}\\
            &\leq {\Big(1+\beta(\alpha\Tilde{S}-S)+\frac{\beta^2}{2}\mathbb{E}\big[{(\alpha\Delta p_{ij}-S)}^2\big]\Big)}^{\frac{t(t-1)}{2}}\\
            &\leq {\Big(1+\beta(\alpha\Tilde{S}-S)+\beta^2\big(\alpha^2\sigma^2+{(\alpha\Tilde{S}-S)}^2\big)\Big)}^{\frac{t(t-1)}{2}}\\
            &\leq e^{\frac{t(t-1)}{2}\beta(\alpha\Tilde{S}-S)+\frac{t(t-1)}{2}\beta^2\big(\alpha^2\sigma^2+{(\alpha\Tilde{S}-S)}^2\big)}
        \end{aligned}
    \end{equation}
    Similarly, we have:
    \begin{equation}
    \label{eqemi2}
        \mathbb{E}\big[e^{-\frac{\beta t(t-1)r(S)}{2}}\big] \leq e^{-\frac{t(t-1)}{2}\beta(\alpha\Tilde{S}-S)+\frac{t(t-1)}{2}\beta^2\big(\alpha^2\sigma^2+{(\alpha\Tilde{S}-S)}^2\big)}.
    \end{equation}
    According to Eq.\ref{eqemi1} and Eq.\ref{eqemi2}, we have that for any $\epsilon\in(0,1/2)$,  there exists:
    \begin{equation}
    \label{bound1}
        B_+(S) = \alpha\Tilde{S}-S+\beta\big(\alpha^2\sigma^2+{(\alpha\Tilde{S}-S)}^2\big)+\frac{2log(\epsilon^{-1})}{t(t-1)\beta}.
    \end{equation}
    \begin{equation}
    \label{bound2}
        B_-(S) = \alpha\Tilde{S}-S-\beta\big(\alpha^2\sigma^2+{(\alpha\Tilde{S}-S)}^2\big)+\frac{2log(\epsilon^{-1})}{t(t-1)\beta}.
    \end{equation}
    By Markov inequality and Eq.\ref{bound1} and Eq.\ref{bound2}, we have:
    \begin{equation}
    \begin{aligned}
        \mathbb{P}\big(r(S)\geq B_+(S)\big)
        &=\mathbb{P}\big(e^{\frac{\beta t(t-1)r(S)}{2}}\geq e^{\frac{\beta t(t-1)B_+(S)}{2}}\big)\\
        &\leq \frac{\mathbb{E}\big[e^{\frac{\beta t(t-1)r(S)}{2}}\big]}{e^{\frac{t(t-1)}{2}\beta(\alpha\Tilde{S}-S)+\frac{t(t-1)}{2}\beta^2\big(\alpha^2\sigma^2+{(\alpha\Tilde{S}-S)^2\big)}+log(\epsilon^-1)}}\\
        &\leq \frac{e^{\frac{t(t-1)}{2}\beta(\alpha\Tilde{S}-S)+\frac{t(t-1)}{2}\beta^2\big(\alpha^2\sigma^2+{(\alpha\Tilde{S}-S)^2\big)}}}{e^{\frac{t(t-1)}{2}\beta(\alpha\Tilde{S}-S)+\frac{t(t-1)}{2}\beta^2\big(\alpha^2\sigma^2+{(\alpha\Tilde{S}-S)^2\big)}+log(\epsilon^-1)}}=\epsilon.
    \end{aligned}
    \end{equation}
    \vspace{0.1cm}
    Thus, we have:
    \begin{equation}
        \mathbb{P}\big(r(S)\le B_+(S)\big) \geq 1-\epsilon.
    \end{equation}
    Similarly,
    \begin{equation}
        \mathbb{P}\big(r(S)\ge B_-(S)\big) \geq 1-\epsilon.
    \end{equation}
    Thus, we can claim:
    \begin{equation}
    \label{pac1}
        \mathbb{P}\big(B_-(S) \le r(S)\le B_+(S)\big) \geq 1-2\epsilon.
    \end{equation}
    According to \textbf{Lemma 2.3} in \cite{chen2021generalized}, we know that for positive real parameter $\beta$ satisfying:
    \begin{equation}
        0<\beta\leq\frac{\sqrt{\frac{1}{4}-\frac{2log(\epsilon^{-1})}{t(t-1)}}}{\alpha\sigma}.
    \end{equation}
    there exists $S_-$ and $S_+$ that $B_+(S_+)=0$ and $B_-(S_+)=0$, meanwhile, $S_+$ is the smallest solution and $S_-$ is the largest solution. Then, it's easy to derive:
    \begin{equation}
        \mathbb{P}\big(S_-\leq\hat{S}\leq S_+\big) \geq 1-2\epsilon.
    \end{equation}
    since our chosen $\psi(x)$ is a continuous function on $x$ which also means that $r(S)$ is a continuous function on $S$. And we know from Eq.\ref{pac1} when $r(\hat{S})=0$ the following event holds with a probability of at least $1-2\epsilon$:
    \begin{equation}
        S_-\leq\hat{S}\leq S_+.
    \end{equation}
    Following the \textbf{Theorem2.6} in \cite{chen2021generalized}, we denote $\beta=\frac{\sqrt{\frac{2log(\epsilon^{-1})}{t(t-1)}}}{\alpha\sigma}$, $n\geq(2\alpha^2\sigma^2+1)^2log(\epsilon^{-1})/{\alpha^2\sigma^2}$. When the difference between \vspace{0.1cm} $S_-$ and $S_+$ is small we can derive the estimator can be localized in a small interval, which implies:
    \begin{equation}
        |\alpha\Tilde{S}-\hat{S}|<\frac{2\alpha\sigma\sqrt{\frac{2log(\epsilon^{-1})}{t(t-1)}}}{1-\sqrt{\frac{2log(\epsilon^{-1})}{t(t-1)\alpha^2\sigma^2}}}=\textit{O}\big(({log(\epsilon^{-1}}))^\frac{1}{2}t^{-1}\big).
    \end{equation}
    holds with a probability of at least $1-2\epsilon$.
\end{proof}

\begin{figure}[ht]
    \centering
    \includegraphics[scale=0.8]{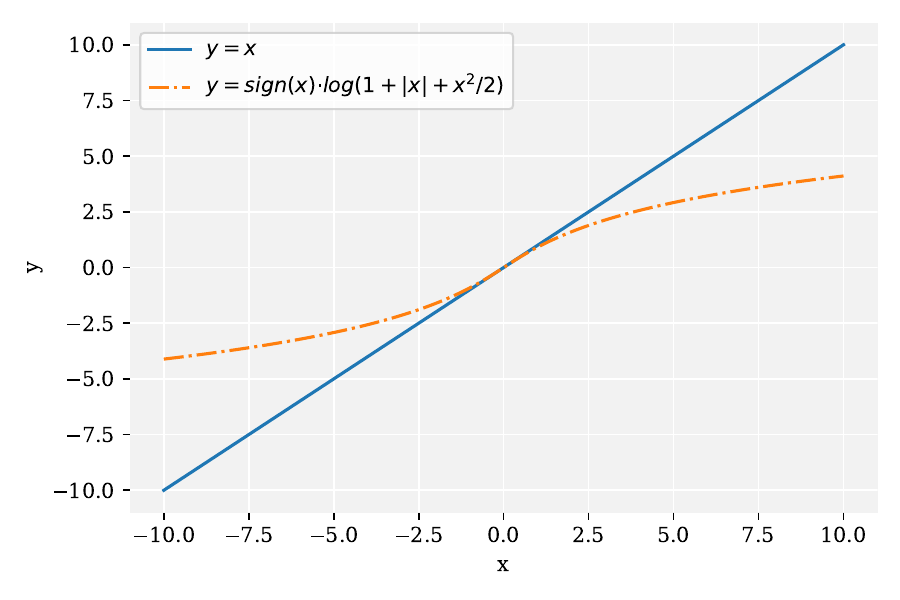} 
    \caption{ The illustration of our proposed robust mean estimator to assess the model's predictive trend.}
    \label{estimator}
\end{figure}
After the theoretical analysis, we present a graph of our robust mean estimator, which sheds light on its underlying mechanism. As illustrated in Figure\ref{estimator}, the estimator is less sensitive to outliers and deviations from normality when the input value $x$ is too large or too small, as indicated by the flatter curve of $f(x)$ in its head and tail. Furthermore, the scaling parameter $\alpha$ enhances the flexibility of the estimator in handling extreme scenarios.

\vspace{-10pt}

\section{Fisher-Jenks Natural Break Classification}
\label{Natural Breaks Classification}

\begin{algorithm}[htb] 
\caption{Fisher (Jenks) Natural Break by Dynamic Programing}
\label{alg_fisher}
\textbf{Input:} Sequence of \textbf{trend score} values $x_i$ for $i\in{1,...,N}$ \\
\textbf{Output:} Class-break index $b$
\begin{algorithmic}
    \STATE \textbf{sort} $\mathcal{X}=\{x_i, 1\leq i\leq N\}$ to a strictly increasing sequence.
    \STATE $\sigma^{2+}_1 \gets 0$; $\Bar{X}_1^+ \gets x_1$; $\sigma^{2-}_N- \gets 0$; $\Bar{X}_1^- \gets x_n$; $b \gets 0$; $s \gets \infty$
    \FOR{$n=2$ to $N$}
        \STATE $\Bar{X}_{n}^+ = \frac{1}{n}x_n+\frac{n-1}{n}\Bar{X}^+_{n-1}$
        \STATE $\sigma^{2+}_n = \frac{n-2}{n-1}\sigma^{2+}_{n-1} + \frac{1}{n}(\Bar{X}_{n}^+-\Bar{X}_{n-1}^+)^2$
    \ENDFOR
    \FOR{$n=N-1$ to $1$}
        \STATE $\Bar{X}_{n}^- = \frac{1}{n}x_n+\frac{n-1}{n}\Bar{X}_{n-1}^-$
        \STATE $\sigma^{2-}_n = \frac{n-2}{n-1}\sigma^{2-}_{n-1} + \frac{1}{n}(\Bar{X}_{n}^--\Bar{X}_{n-1}^-)^2$
    \ENDFOR
    \FOR{$n=1$ to $N-1$}
        \IF{$\sigma^{2-}_{n+1} + \sigma^{2+}_n < s$}
            \STATE $s=\sigma^{2-}_{n+1} + \sigma^{2+}_n; b=n$
        \ENDIF
    \ENDFOR
    \STATE \textbf{return} {$b$};
\end{algorithmic}
\end{algorithm}

In this section, we provide a specific training procedure for finding the Fisher Jenks Natural Break point in a binary scenario. As outlined in Algorithm\ref{alg_fisher}, the sorting process is simple and can be implemented using any sorting algorithm with a worst-case time complexity of $\textit{O}\big(Nlog(N)\big)$. Then, we use a recursive approach to compute the mean and variance of the sequence in both ascending and descending orders. This enables us to obtain a chart for the sum of variances for every possible split, with a time complexity of $\textit{O}\big(N\big)$. Therefore, the overall time complexity remains $\textit{O}\big(Nlog(N)\big)$. Compared with the original algorithm of finding the Fisher Natural Break Point that asks for a time complexity of  $\textit{O}\big(N^2\big)$. Afterward, we provide a detailed derivation of our recursive method for computing the mean and variance. We take the variance $\sigma^{2+}_n$ in ascending order as an example:
\begin{equation}
    \Bar{X}_{n} = \frac{1}{n}x_n+\frac{n-1}{n}\Bar{X}_{n-1}
\end{equation}
\begin{equation}
    \sigma^2_n = \frac{1}{n-1} \sum^n_{i=1}(x_i-\Bar{X}_n)^2 = \frac{1}{n-1} \sum^n_{i=1}\big[(x_i-\Bar{X}_{n-1}) + (\Bar{X}_{n-1}-\Bar{X}_n)\big]^2.
\end{equation}
where $\Bar{X}_n$ is the averaged value of the first $n$ values in the sequence. Then, we can have:
\begin{equation}
    \begin{aligned}
        (n-1)\sigma^2_n
        &=\sum^n_{i=1}\big[(x_i-\Bar{X}_{n-1})^2 + (\Bar{X}_{n-1}-\Bar{X}_n)^2 + 2(x_i-\Bar{X}_{n-1})(\Bar{X}_{n-1}-\Bar{X}_n)\big]\\
        &=\sum^n_{i=1}(x_i-\Bar{X}_{n-1})^2 + \sum^n_{i=1}(\Bar{X}_{n-1}-\Bar{X}_n)^2 + 2\sum^n_{i=1}(x_i-\Bar{X}_{n-1})(\Bar{X}_{n-1}-\Bar{X}_n)\\
        &=\sum^{n-1}_{i=1}(x_i-\Bar{X}_{n-1})^2 + (x_n-\Bar{X}_{n-1})^2 + n(\Bar{X}_{n-1}-\Bar{X}_n)^2+\\ 
        &2(\Bar{X}_{n-1}-\Bar{X}_n)\sum^n_{i=1}(x_i-\Bar{X}_{n-1})\\
        &=(n-2)\sigma^2_{n-1} + (x_n-\Bar{X}_{n-1})^2 + n(\Bar{X}_{n-1}-\Bar{X}_n)^2+\\
        &2(\Bar{X}_{n-1}-\Bar{X}_n) \big[\sum^{n-1}_{i=1}(x_i-\Bar{X}_{n-1}) + (x_n-\Bar{X}_{n-1})\big]\\
        &=(n-2)\sigma^2_{n-1} + (x_n-\Bar{X}_{n-1})^2 + n(\Bar{X}_{n-1}-\Bar{X}_n)^2+\\
        &2(\Bar{X}_{n-1}-\Bar{X}_n)(x_n-\Bar{X}_{n-1})\\
        &=(n-2)\sigma^2_{n-1} + (x_n-\Bar{X}_{n-1})^2 + n(\Bar{X}_{n-1}-\Bar{X}_n)^2-2n(\Bar{X}_{n-1}-\Bar{X}_{n})^2\\
        &=(n-2)\sigma^2_{n-1} + (n^2-n)(\Bar{X}_{n-1}-\Bar{X}_{n})^2\\
        &=(n-2)\sigma^2_{n-1}+\frac{n-1}{n}(x_n-\Bar{X}_{n-1})^2.
    \end{aligned}
\end{equation}
Similarly, the variance $\sigma^{2-}_n$ in descending order can be calculated in a similar way. Then, it's natural for us to have a chart for the sum of variances for every possible split from which the Fisher Natural break point is available.

\section{Additional Experiments}
Here we discuss additional results in other practical settings and further demonstrate the robustness of our method. As mentioned in Section\ref{early stopping}, the model witnesses a dramatic performance degradation when positive data occupies a majority of the unlabeled set or the SCAR (selected completely at random) assumption is violated but such data scenarios are widespread in real-world applications. Moreover, we also make some brief comparisons with other methods under more complex backbones with a varying number of positive labels. 

\begin{table}[htbp] 
\caption{Results of classification accuracy ($\%$) on CIFAR10-1 wiht varying number of postive labels under different backbones (ResNet18 and CNN7 as the backbone model).}
\centering
\setlength\tabcolsep{16pt}
{
    \resizebox{0.9\textwidth}{!}
    {
        \begin{tabular}{cccccc}
        \toprule[1.5pt]
            Backbone                    & Algorithm  & $n_p$ = 0.5k     & $n_p$ = 1k        & $n_p$ = 3k    & $n_p$ = 10k    \\ \midrule[1pt]
        \multirow{6}{*}{CNN7}           & Resampling & 86.29            & 90.02             & 92.64         & 93.41 \\
                                        & uPU        & 82.49           & 76.52             & 87.34         & 93.02 \\
                                        & nnPU       & 85.11            & 84.77             & 89.42         & 94.45 \\
                                        & vPU        & 83.05            & 86.74             & 90.54         & \textbf{95.99} \\ 
                                        & Dist-PU    & 85.15            & 87.25             & 91.76         & 95.07 \\ 
                                        & Ours & \textbf{87.21}   & \textbf{90.58}    &\textbf{91.80} & 95.94\\ \midrule[1pt]
        \multirow{6}{*}{ResNet18}       & Resampling & 84.27            & 88.32             & 90.21         & 93.88 \\
                                        & uPU        & 84.78            & 86.94             & 89.72         & 92.75 \\
                                        & nnPU       & 86.05            & 89.43             & 90.01         & 91.84 \\
                                        & vPU        & 71.40            & 86.85             & 88.54         & 89.89 \\ 
                                        & Dist-PU    & 92.15            & 92.94             & 93.47         & \textbf{96.77} \\ 
                                        & Ours & \textbf{93.21}   & \textbf{94.58}    &\textbf{95.77} & 96.44\\ \bottomrule[1.5pt]
        \end{tabular}
    }
}
\vspace{0.1cm}

\label{tab_labelnum&backbone}
\end{table}
Based on Table\ref{tab_labelnum&backbone}, the Trend-based PU framework performs better in scenarios where the number of positive labels is limited. This could be attributed to the fact that when there are 10,000 positive labels available, the estimation bias and prediction errors caused by the negative assumption are reduced due to the ample availability of supervised information. For imbalanced data, we give different imbalanced divisions compared with ImbalancedPU \cite{imbpu} by following the practice of long-tailed recognition. 10 different categories of CIFAR-10 are distributed under an exponential function with imbalance ratios $\gamma$ in $\{10,100,1000\}$ (the ratio of most populated class to least populated) and we follow the division above in Appendix\ref{implementaion detail} to form the positive and negative set respectively. Thus, the positive prior $\pi$ also gets fixed when the head class is determined as positive or negative. Compared with the division in ImbalancedPU that only choose one category as a positive class, our proposed one is more practical and challenging since it is common practice for a positive class to have different classes with an imbalanced number of data. Besides, in this case, the labeled data and positive data in the unlabeled set share different distributions which do not align with the common SCAR assumption. While our method also gets challenged when negative examples is rare, it still presents much better performance. Actually, when we look into this problem that the majority of unlabeled data is positive or negative. It even makes PUL two completely different questions, 

\begin{table}[htbp]
\caption{Results of classification accuracy(ACC), AUC and F1 score ($\%$) on test set with same number of labels ($1000$) but varying positive prior. }
\centering
\resizebox{1\textwidth}{!}
{
    \begin{tabular}{ccccccccccccc}
    \toprule[1.5pt]
    \multirow{2}{*}{Method} & \multicolumn{3}{c}{$\pi=0.124$, $\gamma=1000$} & \multicolumn{3}{c}{$\pi=0.712$, $\gamma=10$} & \multicolumn{3}{c}{$\pi=0.888$, $\gamma=100$} & \multicolumn{3}{c}{$\pi=0.960$, $\gamma=1000$}  \\ \cmidrule{2-13} 
                            & ACC     & AUC    & F1     & ACC     & AUC    & F1     & ACC     & AUC    & F1     & ACC     & AUC    & F1     \\ \midrule[1pt]
    Resampling             & 92.05   & 96.41  & 91.45  & 74.13   & 82.32  & 42.10   & 70.40   & 79.45  & 35.31  & 67.24    & 71.90  & 14.11   \vspace{0.1cm} \\ 
    ImbPU           & \textbf{92.61}   & 97.12  & 92.51    & 83.22   & \textbf{93.15}  & 86.11  & 74.12   & 84.58  & 77.25  & 71.27   & 80.31  &65.47  \\ \midrule[1pt]
    Ours              & 92.52   & 96.60  & \textbf{92.80}   & \textbf{83.57}   & 90.84  & \textbf{86.85}  & \textbf{80.01}   & \textbf{90.02}  & \textbf{84.68}  & \textbf{75.35}   & \textbf{88.51}  & \textbf{80.72}    \\ \bottomrule[1.5pt]
    \end{tabular}
}
\vspace{0.1cm}

\label{tab_imbalance}
\end{table}
 We compare our method with the resampling baseline and ImbalancedPU specially designed for imbalanced distributions based on popular nnPU and uPU. The results of accuracy, AUC and F1 score on the test set are given in Table \ref{tab_imbalance}. We denote the $\pi$ as the positive prior of the whole dataset including the labeled data. It has illustrated that traditional cost-sensitive based methods can make competitive performance when the data distribution is balanced or positive class is rare. However, it witnesses a significant descent on all three metrics when the majority of unlabeled data belongs to the positive class and we argue that such a situation is quite common especially in the case when positive data is easy to obtain.

\begin{algorithm}[tb]
\caption{Training procedure of the proposed method}
\label{alg1}
\textbf{Input}: positive set $\mathcal{P}$, unlabeled set $\mathcal{U}$\\ 
\textbf{Parameter}: scaling parameter $\alpha$, evaluation step $q$\\
\textbf{Output}:  model parameters $\Theta$
\begin{algorithmic}[1] 
\STATE \textbf{Initialize $\Theta$}, $t=0$ and translate the unlabeled set $\mathcal{U}$ into negative set by negativity assumption;
\WHILE{$t$ $\le$ $MaxEpoch$}
    \STATE Shuffle $\mathcal{P} \cup \mathcal{U}$ into $I$ mini-batches and denote the $i$-th mini-batch as $(\mathcal{B}_p^i,\mathcal{B}_u^i)$;
    \FOR{$i=1$ to $q$}
        \STATE Compute the loss via Eq.\ref{loss}
        \STATE update model parameters $\Theta$ with Adam;
    \ENDFOR
    \STATE Record the model's predictions on the unlabeled set $\mathcal{D}_t=\{p_1, p_2, \ldots, p_{|\mathcal{U}|}\}$
\ENDWHILE
\FOR{$i=1$ to $|\mathcal{U}|$}
    \STATE calculate the \textbf{trend score} $s_i$ on $\mathcal{D}$ through Eq.\ref{eq_trendscore} or Eq.\ref{eq_simplified_trendscore}.
\ENDFOR
\STATE Split the the unlabeled set $\mathcal{U}$ by Algorithm\ref{alg_fisher} to get reformalized positive set $\mathcal{P}$ and negative set $\mathcal{N}$
\STATE \textbf{Reinitialize $\Theta$} and train a binary model on the new positive set $\mathcal{P}$ and negative set $\mathcal{N}$
\STATE \textbf{return} model parameters $\Theta$
\end{algorithmic}
\end{algorithm}

\section{Implementation details}
\label{implementaion detail}

The detailed description of these benchmark datasets is given in Table \ref{tab_dataset} and we denote the category labels with integers ranging from 0 to 9 following the default settings in torchvision. For each dataset, we split the dataset into two disjoint sets as positive and negative following the protocol of \cite{vpu}. Specifically, the labels are defined as follows:  F-MNIST-1: “0,2,4,7” vs “1,5,6,8,9”, F-MNIST-2: “1,5,6,8,9” vs “0,2,4,7”; CIFAR-10-1: “0,1,8,9” vs “2,3,4,5,6,7”, CIFAR-10-2: “2,3,4,5,6,7” vs“0,1,8,9”; STL-10-1: “0,2,3,8,9” vs “1,4,5,6,7”, STL-10-2: “1,4,5,6,7” vs “0,2,3,8,9”; Credit Fraud: "Fraud" vs "Non-Fraud; Alzheimer: "Demented" vs "Non-Demented". 

For a fair comparison, we generally follow the experimental settings as \cite{mixpul,zhaodist}. Specifically, we use the same data split as \cite{mixpul} in CIFAR-10-1, CIFAR-10-2, STL-10-1, STL-10-1 and Credit Card. For Alzheimer, F-MNIST-1 and F-MNIST-2, we follow the settings of \cite{zhaodist}. To verify the effectiveness of our proposed method, We compare our method with several competitive PUL algorithms including uPU\cite{duupu}, nnPU\cite{kiryonnpu}, RP\cite{northcutt2017learning} nnPU with the mixup regularization term, Self-PU\cite{selfpu}, PUSB\cite{kato2019learning}, PUbN\cite{hsieh2019classification}, aPU\cite{hammoudeh2020learning}, vPU\cite{vpu}, MIXPUL\cite{mixpul}, PAN\cite{pgan}, PULNS \cite{pulns}, Dist-PU\cite{zhaodist} and P$^3$MIX \cite{liyour}. For the methods requiring the positive prior, we provide them with an accurate prior except for STL since the true positive prior for STL is actually unknown considering it contains "real" unlabeled data. To this end, we estimate the positive prior of STL by KM2 method\cite{km} before evaluating these methods. We report the results of these datasets under the backbones detailed in Table\ref{tab_dataset} which is identical with \cite{mixpul}. It is worth mentioning that the true labels of unlabeled data in STL10 are not available and that's the reason why we do not report any evaluation of the classification on the unlabeled training data in STL10. We run our method five times, following the procedure of \cite{mixpul}, and report the average metrics and their standard deviations. 

Furthermore, for the results presented in Table \ref{tab_transductive}, we evaluate the key metrics of existing PUL methods based on their predictions on the unlabeled set, which can be considered as a transductive experimental setting. Specifically, we report the recall rate for the Credit Card dataset and the accuracy for the remaining datasets. For Table \ref{tab_prior estimate}, we compare the estimated priors of our method with those of other state-of-the-art prior estimation methods. Although our method is not designed for prior estimation, the positive prior is naturally available when the classification of unlabeled data is performed.

In most cases, we perceive accuracy as the most important evaluation metric except for Credit Fraud dataset. In fraud detection, recall is often more important than precision or accuracy because the consequences of missing a fraudulent transaction can be much more severe than flagging a legitimate transaction as fraudulent. False negatives, which are fraudulent transactions that go undetected, can result in significant financial losses for both the individual and the company. On the other hand, false positives, which are legitimate transactions flagged as fraudulent, may cause temporary inconvenience but can usually be resolved through additional verification steps. Therefore, we emphasize more on recall rate and F1 score on the Credit Fraud dataset. 

While existing Positive and Unlabeled Learning (PUL) methods mainly adopt an inductive learning paradigm, we have observed that some literature fails to report the hyperparameter tuning and model selection process. In traditional machine learning, researchers typically perform these tasks on an independent validation set, but this strategy may not be feasible in PUL due to the lack of negative data. While we can still use an extra positive set as a validation set, in real-world scenarios, the number of labeled data may be limited, especially for PUL paradigms. Furthermore, estimates made under such settings may be conservatively biased due to the limited number of data, particularly for small-scale validation sets. Instead of holding out data, we propose to perform model selection on an augmented validation set using mix-up techniques. Our approach yields comparable results to using an auxiliary positive validation set, as demonstrated in Table \ref{tab_transductive} and Table \ref{tab_resample}. In our comparison, we follow the settings in \cite{mixpul} and hold out 500 positive examples as a validation set. However, we use an augmented validation set based on mix-up techniques and the original labeled training set available to choose the stopping iteration to form our \textbf{trend score}.

For detailed experimental settings, we set the batch size to 64 and the evaluation step to 512 for all datasets and settings. The learning rate is set to 0.0015 for CIFAR10-1, CIFAR10-2, and STL10-1, 0.001 for STL10-2, and 0.002 for Credit Card and Alzheimer datasets. All experiments are implemented on RTX2080ti and RTX3080ti.

\section{Future Works}
\subsection{Risk Bound for PUL under SAR Assumption}
In this subsection, we first review the upper and lower risk bound for PUL under the more general SAR assumption derived from \cite{coudray2023risk}. Compared with the SCAR assumption that assumes the probability for a positive instance to be labeled
is constant and thus independent from the covariates, a more general case is to assume the existence of a propensity function $e(x)$:
\begin{equation}
    e(x)=\mathbb{P}(S=1|Y=1,X=x).
\end{equation}
where $S=1$ represents the labeled positive data. Besides, they also assume that the difficulty of the binary classification can be reflected by the \textit{Massart margin} $h$ derived from the regression function $\eta(x)=\mathbb{P}(Y=1|X=x)$:
\begin{equation}
    \exists h>0,\forall x\in \mathbb{R}^d, |2\eta(x)-1|\geq h.
\end{equation}

\begin{lemma}
\label{upper bound}
    Let $\hat{g}$ be a minimizer of the unbiased empirical risk for PUL under the SAR assumption:$\hat{g}\in Argmin_{g\in\mathcal{G}}\hat{R}_n^{SAR}(g)$. Suppose that the \textit{separability} and Massart margin hold, the propensity $e(.)$ is greater than $e_m>0$. Then, we have the following upper bound on the excess risk:
    \begin{equation}
        \mathbb{E}[\ell(\hat{g},g^*)]\leq k_1\Big[min\Big(\frac{V}{ne_mh}\big(1+log(max(1,\frac{nh^2}{V})\big),\sqrt{\frac{V}{ne_m}}\Big)\Big].
    \end{equation}
    where $k_1>0$ is an absolute constant and $V$ is the \textit{Vapnik-Chervonenkis dimension} of $\mathcal{G}$\cite{vapnik1999nature}.
\end{lemma}
\begin{lemma}
\label{lower bound}
    Suppose that $V\leq2$ and $ne_m\geq V$. Let $h'=\sqrt{\frac{V}{ne_m}}$. Keep the assumptions hold in Lemma\ref{upper bound}, $\forall x\in \mathbb{R}^d$, there exists an absolute constant $k_2>0$ such that:\\
    if $h\geq h'$:\\
    \begin{equation}
        \mathcal{R}(\mathcal{G},h)\geq k_2\frac{V-1}{hne_m}.
    \end{equation}
    if $h\leq h'$:\\
    \begin{equation}
        \mathcal{R}(\mathcal{G},h)\geq k_2\sqrt{\frac{V-1}{ne_m}}.
    \end{equation}
\end{lemma}
\subsection{Limitation}
It can be seen from Lemma\ref{upper bound} and Lemma\ref{lower bound} that both bounds depend on $V$, $n$, $h$ and $e_m$. $h$ evaluates the difficulty of the classification task and $e_m$ represents the minimum of the propensity $e(.)$. When we recall the classification results in Table\ref{tab_imbalance} that evaluate the model's performance under various positive class priors. Both our method and the state-of-art PUL method special for imbalanced data witness a significant descent in all three metrics when the majority of unlabeled data belongs to the positive class. It may be explained by both the upper bounds and the lower bounds mentioned above. Specifically, when the majority of unlabeled data belongs to the positive class, $e_m$ gets lower and both the upper and lower bounds in Lemma\ref{upper bound} and Lemma\ref{lower bound} get higher, making the classification more difficult. It asks for a more powerful model for PUL or a new perspective to tackle PUL. As argued in Section\ref{expriment}, the predictive trends derived from the proposed resampling method can be a viable choice for such imbalanced scenarios. However, compared to the existing reweighting methods, the approach based on trend prediction still requires theoretical analysis. In addition, there are more possible methods worth exploring for additional resampling techniques, trend detection, and subsequent classification.


\end{document}